\documentclass{article}

\usepackage[final]{neurips_2023}

\usepackage[utf8]{inputenc} 
\usepackage[T1]{fontenc}    
\usepackage{hyperref}       
\usepackage{url}            
\usepackage{booktabs}       
\usepackage{amsfonts}       
\usepackage{amsmath}
\usepackage{amssymb}
\usepackage{mathtools}
\usepackage{amsthm}
\usepackage{nicefrac}       
\usepackage{microtype}      
\usepackage{xcolor}         
\usepackage{mylatexstyle}
\usepackage{graphicx,subfigure}
\usepackage{changepage}
\usepackage{multirow}
\usepackage{enumitem}
\usepackage{caption}

\usepackage[compact]{titlesec}
\titlespacing*{\section}{0pt}{0pt}{0pt}
\titlespacing*{\subsection}{0pt}{0pt}{0pt}

\usepackage[capitalize,noabbrev]{cleveref}

\usepackage{wrapfig}
\usepackage{xcolor,colortbl}
\definecolor{maroon}{cmyk}{0,0.87,0.68,0.32}
\definecolor{LightCyan}{rgb}{0.88,1,1}
\newcolumntype{a}{>{\columncolor{LightCyan}}c}

\newcommand{\ourmethod}{\textsc{PDE}}

\allowdisplaybreaks
\title{Robust Learning with Progressive Data Expansion Against Spurious Correlation}

\author
{
    Yihe Deng\thanks{Equal contribution}
    ~~~
    Yu Yang\footnotemark[1]
    ~~~
    Baharan Mirzasoleiman
    ~~~
    Quanquan Gu\\
    Department of Computer Science\\ University of California, Los Angeles\\
    Los Angeles, CA 90095 \\
    \texttt{\{yihedeng,yuyang,baharan,qgu\}@cs.ucla.edu} 
}

\begin{document}

\maketitle

\begin{abstract}
While deep learning models have shown remarkable performance in various tasks, they are susceptible to learning non-generalizable \textit{spurious features} rather than the core features that are genuinely correlated to the true label.
In this paper, beyond existing analyses of linear models, we theoretically examine the learning process of a two-layer nonlinear convolutional neural network in the presence of spurious features. 
Our analysis suggests that imbalanced data groups and easily learnable spurious features can lead to the dominance of spurious features during the learning process.
In light of this, we propose a new training algorithm called \textbf{\ourmethod} that efficiently enhances the model's robustness for a better worst-group performance.
{\ourmethod} begins with a group-balanced subset of training data and progressively expands it to facilitate the learning of the core features. 
Experiments on synthetic and real-world benchmark datasets confirm the superior performance of our method on models such as ResNets and Transformers. 
On average, our method achieves a $2.8\%$ improvement in worst-group accuracy compared with the state-of-the-art method, while enjoying up to $10\times$ faster training efficiency. Codes are available at \url{https://github.com/uclaml/PDE}.
\end{abstract}

\section{Introduction}
\begin{figure}[!ht]
\centering
\includegraphics[width=0.9\textwidth]{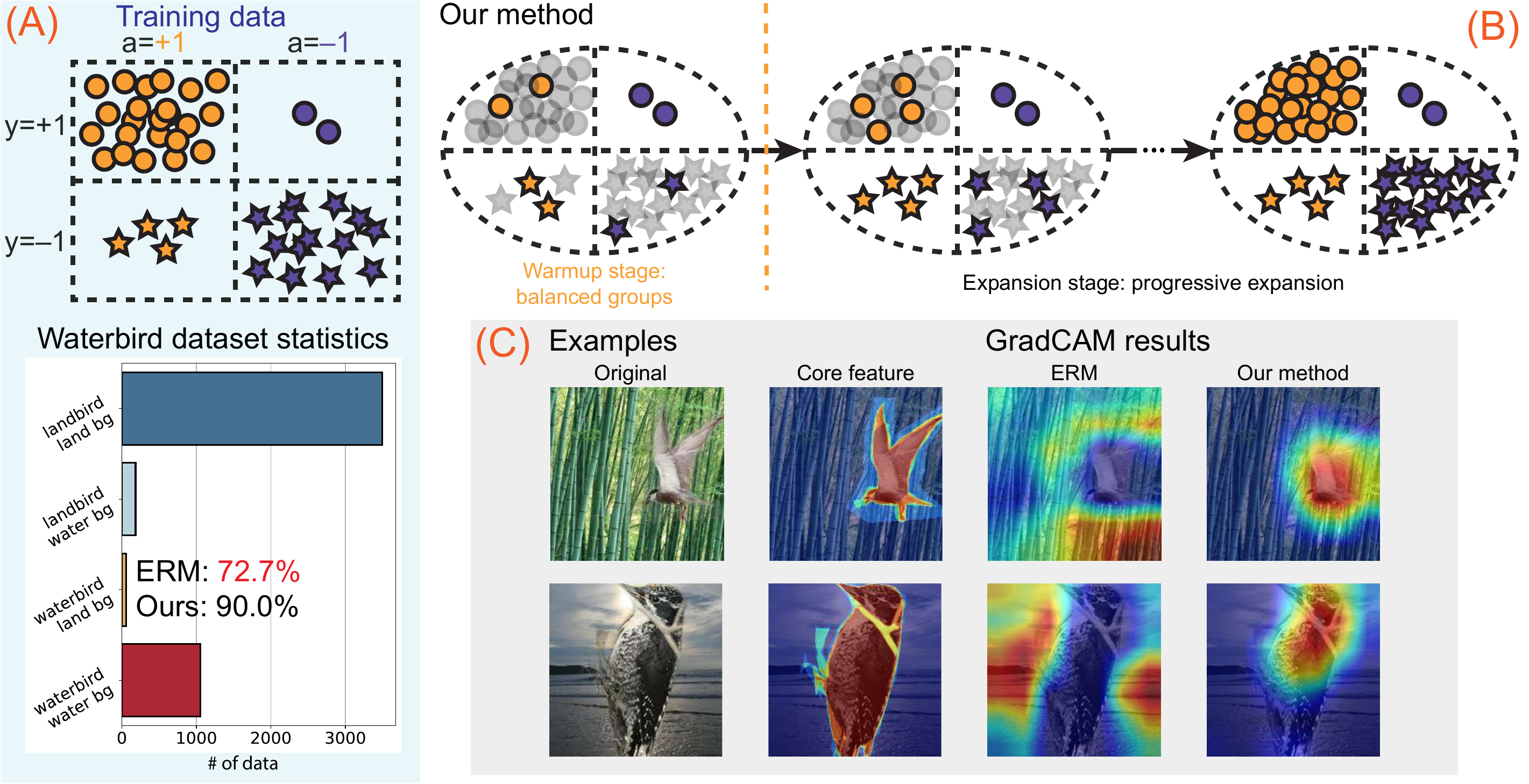}
\caption{An overview of the problem, our proposed solution, and the resultant outcomes. (A) We demonstrate the data distribution and provide an example of the statistics of Waterbirds. (B) The overall procedure of PDE. (C) We use GradCAM~\citep{selvaraju2017grad} to show the attention of the model trained with {\ourmethod} as compared to ERM. While ERM focuses on the background, PDE successfully trains the model to capture the birds.}\vspace{-1cm}
\label{fig:demo}
\end{figure}
Despite the remarkable performance of deep learning models, recent studies~~\citep{sagawa2019distributionally,sagawa2020investigation,izmailov2022feature,haghtalab2022ondemand,yang2022understanding,yang2023mitigating,yang2023change,joshi2023towards} have identified their vulnerability to spurious correlations in data distributions. A spurious correlation refers to an easily learned feature that, while unrelated to the task at hand, appears with high frequency within a specific class. 
For instance, waterbirds frequently appear with water backgrounds, and landbirds with land backgrounds. 
When training with empirical risk minimization (ERM), deep learning models tend to exploit such correlations and fail to learn the more subtle features genuinely correlated with the true labels, resulting in poor generalization performance on minority data (e.g., waterbirds with land backgrounds as shown in Figure~\ref{fig:demo}).
This observation raises a crucial question: 
\textit{Does the model genuinely learn to classify birds, or does it merely learn to distinguish land from water?} 
The issue is particularly concerning because deep learning models are being deployed in critical applications such as healthcare, finance, and autonomous vehicles, where we require a reliable predictor.

Researchers formalized the problem by considering examples with various combinations of core features (e.g., landbird/waterbird) and spurious features (e.g., land/water backgrounds) as different \textit{groups}. 
The model is more likely to make mistakes on certain groups if it learns the spurious feature. The objective therefore becomes balancing and improving performance across all groups. 
Under this formulation, we can divide the task into two sub-problems: (1) accurately identifying the groups, which are not always known in a dataset, and (2) effectively using the group information to finally improve the model's robustness. 
While numerous recent works~\citep{nam2020learning, liu2021just, creager2021environment, ahmed2021systematic,taghanaki2021robust,zhang2022correct} focus on the first sub-problem, the second sub-problem remains understudied. The pioneering work~\citep{sagawa2019distributionally} still serves as the best guidance for utilizing accurate group information. In this paper, we focus on the second sub-problem and aim to provide a more effective and efficient algorithm to utilize the group information. It is worth noting that the theoretical understanding of spurious correlations lags behind the empirical advancements in mitigating spurious features. Existing theoretical studies~\citep{sagawa2020investigation, chen2020self, yang2022understanding,ye2022freeze} are limited to the setting of simple linear models and data distribution that are less reflective of real application scenarios. 

We begin by theoretically examining the learning process of spurious features when training a two-layer nonlinear convolutional neural network (CNN) on a corresponding data model that captures the influence of spurious correlations. 
We illustrate that the learning of spurious features swiftly overshadows the learning of core features from the onset of training when groups are imbalanced and spurious features are more easily learned than core features. 
Based upon our theoretical understanding, we propose \underline{Progressive Data Expansion} (\textbf{\ourmethod}), a neat and novel training algorithm that efficiently uses group information to enhance the model's robustness against spurious correlations. 
Existing approaches, such as GroupDRO~\citep{sagawa2019distributionally} and upsampling techniques~\citep{liu2021just}, aim to balance the data groups in each batch throughout the training process.
In contrast, we employ a small balanced warm-up subset only at the beginning of the training. 
Following a brief period of balanced training,  
we progressively expand the warm-up subset by adding small random subsets of the remaining training data until using all of them, as shown in the top right of Figure~\ref{fig:demo}. 
Here, we utilize the momentum from the warm-up subset to prevent the model from learning spurious features when adding new data. 
Empirical evaluations on both synthetic and real-world benchmark data validate our theoretical findings and confirm the effectiveness of {\ourmethod}. 
Additional ablation studies also demonstrate the significance and impact of each component within our training scheme. 
In summary, our contributions are highlighted as follows:
\begin{itemize}[leftmargin=*,nosep]
\item We provide a theoretical understanding of the impact of spurious correlations beyond the linear setting by considering a two-layer nonlinear CNN. 
\item We introduce {\ourmethod}, a theory-inspired approach that effectively addresses the challenge posed by spurious correlations. 
\begin{itemize}[leftmargin=*,nosep]
    \item {\ourmethod} achieves the best performance on benchmark vision and language datasets for models including ResNets and Transformers. On average, it outperforms the state-of-the-art method by $2.8\%$ in terms of worst-group accuracy.
    \item {\ourmethod} enjoys superior training efficiency, being $10\times$ faster than the state-of-the-art methods.
\end{itemize}
\end{itemize}

\section{Why is Spurious Correlation Harmful to ERM?}
In this section, we simplify the intricate real-world problem of spurious correlations into a theoretical framework. We provide analysis on two-layer nonlinear CNNs, extending beyond the linear setting prevalent in existing literature on this subject. 
Under this framework, we formally present our theory concerning the training process of empirical risk minimization (ERM) in the presence of spurious features. 
These theoretical insights motivate the design of our algorithm. 
\subsection{Empirical Risk Minimization}
We begin with the formal definition of the ERM-based training objective for a binary classification problem. Consider a training dataset $S = \{(\xb_i,y_i)\}_{i=1}^{N}$, where $\xb_i\in\RR^d$ is the input and $y\in\{\pm 1\}$ is the output label. We train a model $f(\xb;\Wb)$ with weight $\Wb$ to minimize the empirical loss function:
\begin{align}
   \cL(\Wb) = \frac{1}{N} \textstyle{\sum_{i=1}^N}\ell\big(y_{i}f(\xb_{i};\Wb)\big),\label{eq:empirical loss}
\end{align}
where $\ell$ is the logistic loss defined as $\ell(z)= \log(1+\exp(-z))$. The empirical risk minimizer refers to $\Wb^*$ that minimizes the empirical loss: $\Wb^*\coloneqq\argmin_\Wb\cL(\Wb)$. 
Typically, gradient-based optimization algorithms are employed for ERM. For example, at each iteration $t$, gradient descent (GD) has the following update rule: 
\begin{align}
    \Wb^{(t+1)}&=\Wb^{(t)}-\eta\nabla\cL(\Wb^{(t)}). \label{eq:GD update}
\end{align}
Here, $\eta>0$ is the learning rate. 
In the next subsection, we will show that even for a relatively simple data model which consists of core features and spurious features, vanilla ERM will fail to learn the core features that are correlated to the true label. 

\subsection{Data Distribution with Spurious Correlation Fails ERM}
Previous work such as \citep{sagawa2020investigation} considers a data model where the input consists of core feature, spurious feature and noise patches at fixed positions, i.e., $\xb=[\xb_{\text{core}},\xb_{\text{spu}},\xb_{\text{noise}}]$. 
In real-world applications, however, features in an image do not always appear at the same pixels. 
Hence, we consider a more realistic data model where the patches do not appear at fixed positions. 

\begin{definition}[Data model]\label{def:data_distribution}
A data point $(\xb,y,a)\in (\RR^{d})^{P} \times\{\pm 1\}\times\{\pm 1\}$ is generated from the distribution $\cD$ as follows. 
\begin{itemize}[leftmargin=*,nosep] 
    \item Randomly generate the true label $y\in\{\pm 1\}$.
    \item Generate spurious label $a\in\{\pm y\}$, where $a=y$ with probability $\alpha>0.5$. 
    \item Generate $\xb$ as a collection of $P$ patches: $\xb = (\xb^{(1)},\xb^{(2)}, \ldots, \xb^{(P)})\in (\RR^{d})^{P}$, where
    \begin{itemize}[leftmargin=*,nosep]
        \item \textbf{Core feature.} One and only one patch is given by $\beta_c\cdot y\cdot \vb_c$ with $\|\vb_c\|_2=1$.
        \item \textbf{Spurious feature.} 
        One and only one patch is given by $\beta_s\cdot a\cdot \vb_s$ with $\|\vb_s\|_2=1$ and $\inner{\vb_{c}}{\vb_{s}}=0$. 
        \item  \textbf{Random noise.} The rest of the $P-2$ patches are Gaussian noises $\bxi$ that are independently drawn from $N(0, (\sigma_{p}^{2}/d)\cdot\Ib_{d})$ with $\sigma_{p}$ as an absolute constant.
    \end{itemize}
    And $0<\beta_c\ll\beta_s \in \RR$. 
\end{itemize} 
\end{definition}
\begin{wrapfigure}{r}{0.35\textwidth}
\centering
\includegraphics[width=0.35\textwidth]{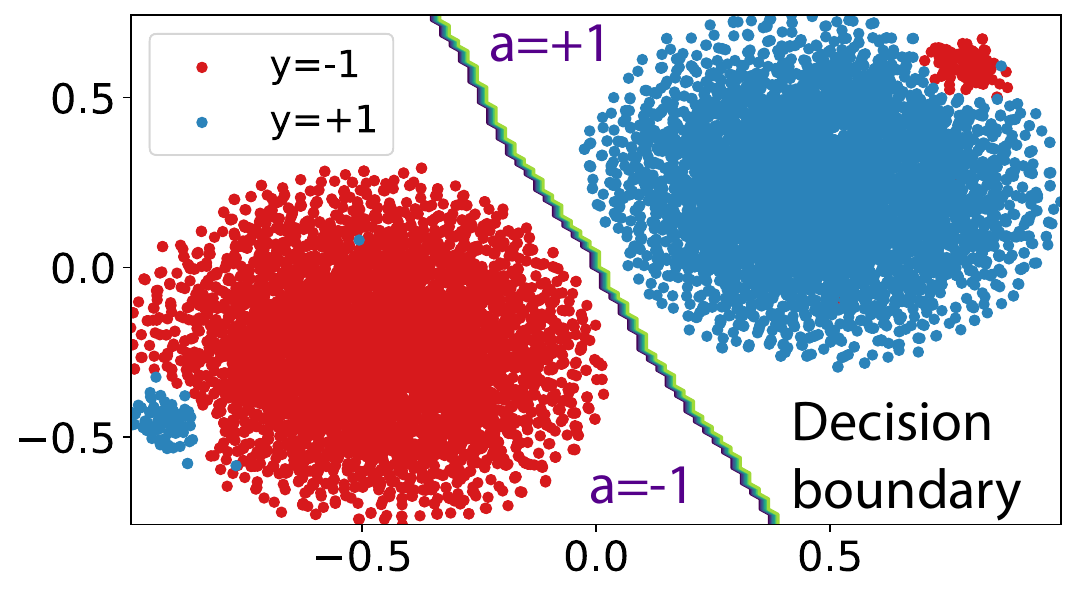}
\caption{\label{fig:data}Visualization of the data.}
\end{wrapfigure}
Similar data models have also been considered in recent works on feature learning~\citep{allen2020towards,zou2021understanding,chen2022towards,jelassi2022towards}, where the input data is partitioned into feature and noise patches. We extend their data models by further positing that certain feature patches might be associated with the spurious label instead of the true label. In the rest of the paper, we assume $P=3$ for simplicity.
With the given data model, we consider the training dataset $S = \{(\xb_i,y_i,a_i)\}_{i=1}^{N}$ and let $S$ be partitioned into large group $S_1$ and small group $S_2$ such that $S_1$ contains all the training data that can be correctly classified by the spurious feature, i.e., $a_i=y_i$, and $S_2$ contains all the training data that can only be correctly classified by the core feature, i.e., $a_i=-y_i$. 
We denote $\hat{\alpha}=\frac{|S_1|}{N}$ and therefore $1-\hat{\alpha}=\frac{|S_2|}{N}$.

\paragraph{Visualization of our data.} 
In Figure~\ref{fig:data}, we present the visualization in $2$D space of the higher-dimensional data generated from our data model using t-SNE~\citep{van2008visualizing}, where data within each class naturally segregate into large and small groups. The spurious feature is sufficient for accurate classification of the larger group data, but will lead to misclassification of the small group data.

\subsection{Beyond Linear Models}
We consider a two-layer nonlinear CNN defined as follows:
\begin{align}
    f(\xb;\Wb) = \textstyle{\sum_{j \in [J]}\sum^{P}_{p=1}} \sigma \big(\langle \wb_{j}, \xb^{(p)} \rangle \big),
    \label{eq:model}
\end{align}
where $\wb_j\in\RR^d$ is the weight vector of the $j$-th filter, $J$ is the number of filters (neurons) of the network, and $\sigma(z)=z^3$ is the activation function. $\Wb=[\wb_{1},\ldots,\wb_{J}]\in\RR^{d\times J}$ denotes the weight matrix of the CNN. 
Similar two-layer CNN architectures are analyzed in in~\citep{chen2022towards,jelassi2022towards} but for different problems, where the cubic activation serves a simple function that provides non-linearity. 
Similar to \citet{jelassi2022towards,cao2022benign}, we assume a mild overparameterization of the CNN with $J=\text{polylog}(d)$.
We initialize $\Wb^{(0)}\sim \cN(0,\sigma_0^2)$, where $\sigma_0^2=\text{polylog}(d)/d$.
Due to the CNN structure, our analysis can handle data models where each data can have an arbitrary order of patches while linear models fail to do so.

\subsection{Understanding the Training Process with Spurious Correlation}
In this subsection, we formally introduce our theoretical result on the training process of the two-layer CNN using gradient descent in the presence of spurious features. We first define the performance metrics. A frequently considered metric is the test accuracy: $\text{Acc}(\Wb) = \PP_{(\xb,y,a)\sim\cD}\big[\text{sgn}(f(\xb;\Wb))=y\big]$.
With spurious correlations, researchers are more interested in the worst-group accuracy: 
\begin{align*}
\begin{small}
    \text{Acc}_{\textrm{wg}}(\Wb) = \min_{y\in\{\pm1\},a\in\{\pm1\}}\PP_{(\xb,y,a)\sim\cD}\big[\text{sgn}(f(\xb;\Wb))=y\big],
\end{small}
\end{align*}
which accesses the worst accuracy of a model among all groups defined by combinations of $y$ and $a$. 
We then summarize the learning process of ERM in the following theorem. Our analysis focuses on the learning of spurious and core features, represented by the growth of $\la\wb_i^{(t)},\vb_s\ra$ and $\la\wb_i^{(t)},\vb_c\ra$ respectively: 
\begin{theorem}\label{lemma:stage1 spurous}
    Consider the training dataset $S=\{(\xb_i,y_i)\}_{i=1}^{N}$ that follows the distribution in Definition~\ref{def:data_distribution}. Consider the two-layer nonlinear CNN model as in~\eqref{eq:model} initialized with $\Wb^{(0)}\sim \cN(0,\sigma_0^2)$. After training with GD in \eqref{eq:GD update} for $T_0=\tilde{\Theta}\big(1/(\eta\beta_s^3\sigma_0)\big)$ iterations, for all $j\in [J]$ and $t\in[0,T_0)$, we have 
    \begin{align}
        \tilde\Theta(\eta)\beta_s^3(2\hat{\alpha} - 1)\cdot\la\wb_j^{(t)},\vb_s\ra^2
        &\le\la\wb_j^{(t+1)},\vb_s\ra- \la\wb_j^{(t)},\vb_s\ra\le \tilde\Theta(\eta)\beta_s^3\hat{\alpha}\cdot\la\wb_j^{(t)},\vb_s\ra^2,\label{eq:spurious bounds}\\
        \tilde\Theta(\eta)\beta_c^3\hat{\alpha} \cdot\la\wb_j^{(t)},\vb_c\ra^2
        &\le\la\wb_j^{(t+1)},\vb_c\ra- \la\wb_j^{(t)},\vb_c\ra\le \tilde\Theta(\eta)\beta_c^3\cdot\la\wb_j^{(t)},\vb_c\ra^2\label{eq:core bounds}.
    \end{align} 
    After training for $T_0$ iterations, with high probability, the learned weight has the following properties: (1) it learns the spurious feature $\vb_s$: $\max_{j\in[J]}\la\wb_j^{(T)},\vb_s\ra \ge \tilde{\Omega}(1/\beta_s)$; (2) it \textit{almost} does not learn the core feature $\vb_c$: $\max_{j\in[J]}\la\wb_j^{(T)},\vb_c\ra = \tilde{\cO}(\sigma_0)$. 
\end{theorem}
\paragraph{Discussion.} The detailed proof is deferred to Appendix~\ref{appendix:analysis}, 
and we provide intuitive explanations of the theorem as follows. 
A larger value of $\langle\wb_i^{(t)},\vb\rangle$ for $\vb\in\{\vb_s,\vb_c\}$ implies better learning of the feature vector $\vb$ by neuron $\wb_i$ at iteration $t$. 
As illustrated in \eqref{eq:spurious bounds} and~\eqref{eq:core bounds}, the updates for both spurious and core features are non-zero, as they depend on the squared terms of themselves with non-zero coefficients, while the growth rate of $\langle\wb_i^{(t)},\vb_s\rangle$ is significantly faster than that of $\langle\wb_i^{(t)},\vb_c\rangle$. 
Consequently, the neural network rapidly learns the spurious feature but barely learns the core feature, as it remains almost unchanged from initialization as compared to the spurious feature.

We derive the neural network's prediction after training for $T_0$ iterations. 
For a randomly generated data example $(\xb, y, a) \sim \mathcal{D}$, the neural network's prediction is given by $\text{sgn}\big(f(\xb;\Wb)\big)=\text{sgn}\big(\sum_{j\in[J]}\big(y\beta_c^3\langle\wb_j,\vb_c\rangle^3+a\beta_s^3\langle\wb_j,\vb_s\rangle^3+\langle\wb_j,\xi\rangle^3\big)\big)$. 
Since the term $\beta_s^3\max_{j\in[J]}\langle\wb_j,\vb_s\rangle^3$ dominates the summation, the prediction will be $\text{sgn}(f(\xb;\Wb))=a$. 
Consequently, we obtain the test accuracy as $\text{Acc}(\Wb)=\alpha$, since $a=y$ with probability $\alpha$, and the model accurately classifies the large group. 
However, when considering the small group and examining examples where $y\ne a$, the models consistently make errors, resulting in $\text{Acc}_{wg}(\Wb)=0$. 
To circumvent this poor performance on worst-group accuracy, an algorithm that can avoid learning the spurious feature is in demand. 

\section{Theory-Inspired Two-Stage Training Algorithm} 
In this section, we introduce Progressive Data Expansion ({\ourmethod}), a novel two-stage training algorithm inspired by our analysis to enhance robustness against spurious correlations. We begin with illustrating the implications of our theory, where we provide insights into the data distributions that lead to the rapid learning of spurious features and clarify scenarios under which the model remains unaffected.
\subsection{Theoretical Implications}\label{sec:implication}
Notably in Theorem~\ref{lemma:stage1 spurous}, the growth of the two sequences $\la\wb_i^{(t)},\vb_s\ra$ in~\eqref{eq:spurious bounds} and $\la\wb_i^{(t)},\vb_c\ra$ in~\eqref{eq:core bounds} follows the formula $x_{t+1}=x_t+\eta Ax_t^2$, where $x_t$ represents the inner product sequence with regard to iteration $t$ and $A$ is the coefficient containing $\hat{\alpha}$, $\beta_c$ or $\beta_s$. 
This formula is closely related to the analysis of tensor power methods~\citep{allen2020towards}. In simple terms, when two sequences have slightly different growth rates, one of them will experience much faster growth in later times. 
As we will show below, the key factors that determine the drastic difference between spurious and core features in later times are the group size $\hat{\alpha}$ and feature strengths $\beta_c,\beta_s$. 
\begin{itemize}[leftmargin=*,nosep]
    \item \noindent\textbf{When the model learns spurious feature ($\beta_c^3 < \beta_s^3(2\hat{\alpha}-1)$).} We examine the lower bound for the growth of $\la\wb_i^{(t)},\vb_s\ra$ in~\eqref{eq:spurious bounds} and the upper bound for the growth of $\la\wb_i^{(t)},\vb_c\ra$ in~\eqref{eq:core bounds} in Theorem~\ref{lemma:stage1 spurous}.
    If $\beta_c^3 < \beta_s^3(2\hat{\alpha}-1)$, we can employ the tensor power method and deduce that the spurious feature will be learned first and rapidly. The condition on data distribution imposes two necessary conditions: $\hat{\alpha}>1/2$ (groups are imbalanced) and $\beta_c<\beta_s$ (the spurious feature is stronger). 
    This observation is consistent with real-world datasets, such as the Waterbirds dataset, where $\hat{\alpha}=0.95$ and the background is much easier to learn than the intricate features of the birds. 
    \item \noindent\textbf{When the model learns core feature ($\beta_c>\beta_s$).}  However, if we deviate from the aforementioned conditions and consider $\beta_c>\beta_s$, we can examine the lower bound for the growth of $\langle\wb_i^{(t)},\vb_c\rangle$ in~\eqref{eq:core bounds} and the upper bound for the growth of $\langle\wb_i^{(t)},\vb_s\rangle$ in~\eqref{eq:spurious bounds}. Once again, we apply the tensor power method and determine that the model will learn the core feature rapidly. In real-world datasets, this scenario corresponds to cases where the core feature is not only significant but also easier to learn than the spurious feature. Even for imbalanced groups with $\hat{\alpha}>1/2$, the model accurately learns the core feature. 
    Consequently, enhancing the coefficients of the growth of the core feature allows the model to tolerate imbalanced groups. We present verification through synthetic experiments in the next section.
\end{itemize}
As we will show in the following subsection, we initially break the conditions of learning the spurious feature by letting $\hat{\alpha}=1/2$ in a group-balanced data subset. Subsequently, we utilize the momentum to amplify the core feature's coefficient, allowing for tolerance of $\hat{\alpha} > 1/2$ when adding new data.

\subsection{{\ourmethod}: A Two-Stage Training Algorithm}
We present a new algorithm named \underline{Progressive Data Expansion} ({\ourmethod}) in Algorithm~\ref{alg:method} and explain the details below, which consist of (1) warm-up and (2) expansion stages. 

\begin{algorithm}\small 
\caption{Progressive Data Expansion ({\ourmethod})}
\begin{algorithmic}[1]\label{alg:method}
\REQUIRE Number of iterations $T_0$ for warm-up training; number of times $K$ for dataset expansion; number of iterations $J$ for expansion training; number of data $m$ for each expansion; learning rate $\eta$; momentum coefficient $\gamma$; initialization scale $\sigma_{0}$; training set $S = \{(\xb_{i},y_{i},a_{i})\}_{i=1}^{n}$; model $f_\Wb$. 
\STATE Initialize $\Wb^{(0)}$.

\textbf{Warm-up stage}
\STATE Divide the $S$ into groups by values of $y$ and $a$: $S_{y,a}=\{(\xb_{i},y_{i},a_{i})\}_{y_i=y,a_i=a}$.
\STATE Generate warm-up set $S^0$ from $S$ by randomly subsampling from each group of $S$ such that $|S^0_{y,a}|=\min_{y',a'}|S_{y',a'}|$ for $y\in\{\pm1\}$ and $a\in\{\pm1\}$.
\FOR{$t=0,1,\ldots, T_0$}
    \STATE Compute loss on $S^0$: $\cL_{S^0}(\Wb^{(t)})=\frac{1}{|S^0|}\sum_{i\in S^0} \ell(y_if(x_i;\Wb^{(t)}))$.
    \STATE Update $\Wb^{(t+1)}$ by \eqref{eq:momentum update} and \eqref{eq:sgd+m update}. 
\ENDFOR

\textbf{Expansion stage}
\FOR{$k=1,\ldots, K$}
    \STATE Draw $m$ examples ($S_{[m]}$) from $S/S^{k-1}$ and let $S^k=S^{k-1}\cup S_{[m]}$. 
    \FOR{$t=1,\ldots, J$}
    \STATE Compute loss on $S^k$: $\cL_{S^k}(\Wb^{(T)})=\frac{1}{|S^k|}\sum_{i\in S^k} \ell(y_if(x_i;\Wb^{(T)}))$, where $T=T_0+(k-1)*J+t$.
    \STATE Update $\Wb^{(T+1)}$ by \eqref{eq:momentum update} and \eqref{eq:sgd+m update}.
    \ENDFOR
\ENDFOR
\STATE \textbf{return} $\Wb^{(t)}=\argmax_{\Wb^{(t')}}\text{Acc}^{\text{val}}_{\textrm{wg}}(\Wb^{(t')})$.
\end{algorithmic}
\end{algorithm}
As accelerated gradient methods are most commonly used in applications, we jointly consider the property of momentum and our theoretical insights when designing the algorithm.
For gradient descent with momentum (GD+M), at each iteration $t$ and with momentum coefficient $\gamma>0$, it updates as follows 
\begin{align}
    \gb^{(t+1)}&=\gamma \gb^{(t)}+(1-\gamma)\nabla\cL(\Wb^{(t)}), \label{eq:momentum update}\\
    \Wb^{(t+1)}&=\Wb^{(t)}-\eta\cdot\gb^{(t+1)},\label{eq:sgd+m update}
\end{align}

\paragraph{Warm-up Stage.} 
In this stage, we create a fully balanced dataset $S^0$, in which each group is randomly subsampled to match the size of the smallest group, and consider it as a warm-up dataset. We train the model on the warm-up dataset for a fixed number of epochs. During this phase, the model is anticipated to accurately learn the core feature without being influenced by the spurious feature. 
Note that, under our data model, a completely balanced dataset will have $\hat{\alpha}=1/2$. We present the following lemma as a theoretical basis for the warm-up stage. 
\begin{lemma}\label{lemma:motivation}
    Given the balanced training dataset $S^0=\{(\xb_i,y_i,a_i)\}_{i=1}^{N_0}$ with $\hat{\alpha}=1/2$ as in Definition~\ref{def:data_distribution} and CNN as in \eqref{eq:model}. The gradient on $\vb_s$ will be $0$ from the beginning of training.  
\end{lemma}
In particular, with $\hat{\alpha}=1/2$ we have $|S_1^0|=|S_2^0|$: an equal amount of data is positively correlated with the spurious feature as the data negatively correlated with the spurious feature. In each update, both groups contribute nearly the same amount of spurious feature gradient with different signs, resulting in cancellation. Ultimately, this prevents the model from learning the spurious feature. 
Detailed proofs can be found in Appendix~\ref{appendix:motivation}. 

\paragraph{Expansion Stage.} In this stage, we proceed to train the model by incrementally incorporating new data into the training dataset. 
The rationale for this stage is grounded in the theoretical result by the previous work~\citep{jelassi2022towards} on GD with momentum, which demonstrates that once gradient descent with momentum initially increases its correlation with a feature $\vb$, it retains a substantial historical gradient in the momentum containing $\vb$. 
Put it briefly, the initial learning phase has a considerable influence on subsequent training for widely-used accelerated training algorithms. 
While ERM learns the spurious feature $\vb_s$ and momentum does not help, as we will show in synthetic experiments, ${\ourmethod}$ avoids learning $\vb_s$ and learns $\vb_c$ in the warm-up stage.
This momentum from warm-up, in turn, amplifies the core feature that is present in the gradients of newly added data, facilitating the continued learning of $\vb_c$ in the expansion stage. 
For a specific illustration, the learning of the core feature by GD+M will be 
\begin{align*}
    \la\wb_j^{(t+1)},\vb_c\ra
    &=\la\wb_j^{(t)}-\eta \big(\gamma g^{(t)}+(1-\gamma)\nabla_{\wb_j}\cL(\Wb^{(t)})\big),\vb_c\ra,
\end{align*}
where $g^{(t)}$ is the additional momentum as compared to GD with $\gamma=0$. While the current gradient along $\vb_c$ might be small (i.e., $\beta_c$), we can benefit from the historical gradient in $g^{(t)}$ to amplify the growth of $\la\wb_j^{(t+1)},\vb_c\ra$ and make it larger than that of the spurious feature (i.e., $\beta_s$).
This learning process will then correspond to the case when the model learns the core feature discussed in Subsection~\ref{sec:implication}.
Practically, we consider randomly selecting $m$ new examples for expansion every $J$ epochs by attempting to draw a similar number of examples from each group. During the last few epochs of the expansion stage, we expect the newly incorporated data exclusively from the larger group, as the smaller groups have been entirely integrated into the warm-up dataset.  

It is worth noting that while many works address the issue of identifying groups from datasets containing spurious correlations, we assume the group information is known and our algorithm focuses on the crucial subsequent question of optimizing group information utilization. Aiming to prevent the learning of spurious features, {\ourmethod} distinguishes itself by employing a rapid and lightweight warm-up stage and ensuring continuous improvement during the expansion stage with the momentum acquired from the warm-up dataset. Our training framework is both concise and effective, resulting in computational efficiency and ease of implementation.

\section{Experiments}
In this section, we present the experiment results from both synthetic and real datasets. Notably, we report the \underline{worst-group accuracy}, which assesses the minimum accuracy across all groups and is commonly used to evaluate the model's robustness against spurious correlations.

\subsection{Synthetic Data}
\begin{wraptable}{r}{8cm}
\caption{Synthetic data experiments. We report the worst-group accuracy and the gap (i.e., overall - worst). We further consider several variations of PDE to demonstrate the importance of each component of our method. \textbf{Reset}: we reset the momentum to zero after the warm-up stage. \textbf{Warmup+All}: we let PDE incorporate all of the new training data at once after the warm-up stage.}\label{tab:synthetic}
    \begin{tabular}{c c c c}
    \toprule
        & Worst-group ($\%$) & Gap ($\%$) \\
    \midrule
        ERM (GD) & $0.00$ & $97.71$  \\
        ERM (GD+M) & $0.00$ & $97.71$  \\
        Warmup+All (Reset) & $67.69$ & $31.18$ \\
        Warmup+All & $74.24$ & $24.76$ \\
        {\ourmethod} (Reset) & $92.51$ & $2.29$\\
        {\ourmethod} & $\textbf{93.01}$ &  $\textbf{0.03}$ \\
    \bottomrule
    \end{tabular}
\end{wraptable} 
In this section, we present synthetic experiment results in verification of our theoretical findings. In Appendix~\ref{appendix:synthetic}, 
we illustrate the detailed data distribution, hyper-parameters of the experiments and more extensive experiment results. The data used in this section is generated following Definition~\ref{def:data_distribution}. 
We consider the worst-group and overall test accuracy. 
As illustrated in Table~\ref{tab:synthetic}, ERM, whether trained with GD or GD+M, is unable to accurately predict the small group in our specified data distribution where $\hat{\alpha}=0.98$ and $\beta_c<\beta_s$. 
In contrast, our method significantly improves worst-group accuracy while maintaining overall test accuracy comparable to ERM. 
Furthermore, as depicted in Figure~\ref{fig:a}, ERM rapidly learns the spurious feature as it minimizes the training loss, while barely learning the core feature.  
Meanwhile, in Figure~\ref{fig:b} we show the learning of ERM when the data distribution breaks the conditions of our theory and has $\beta_c>\beta_s$ instead. 
Even with the same $\hat{\alpha}$ as in Figure~\ref{fig:a}, ERM  correctly learns the core feature despite the imbalanced group size. 
These two figures support the theoretical results we discussed to motivate our method. 
Consequently, on the same training dataset as in Figure~\ref{fig:a}, Figure~\ref{fig:c} shows that our approach allows the model to initially learn the core feature using the warm-up dataset and continue learning when incorporating new data. 

\begin{figure}[htp]
\centering     
\subfigure[ERM (case 1)]{\label{fig:a}\includegraphics[width=0.3\textwidth]{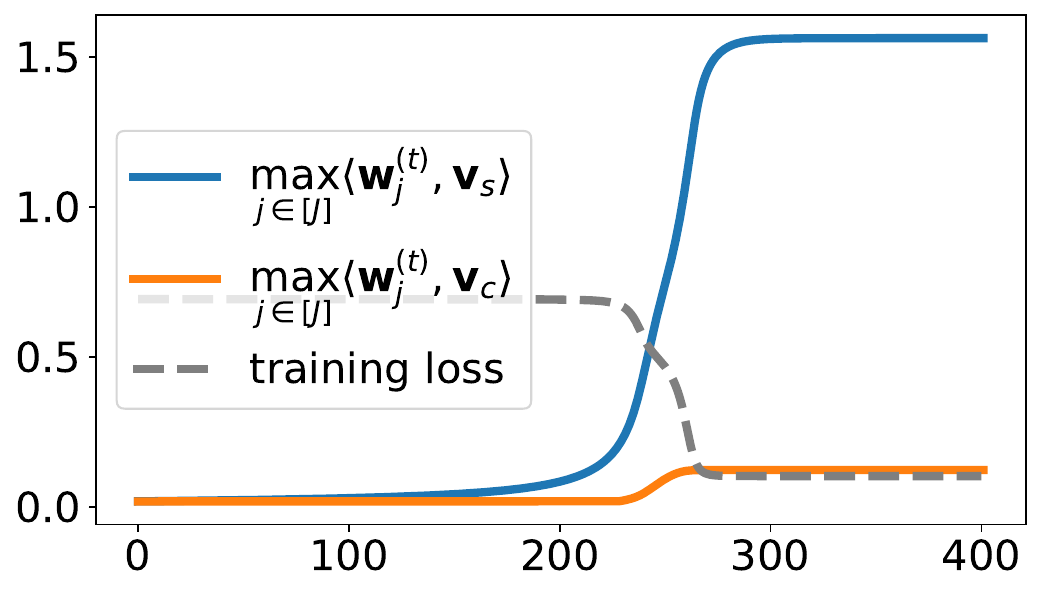}}
\subfigure[ERM (case 2)]
{\label{fig:b}\includegraphics[width=0.3\textwidth]{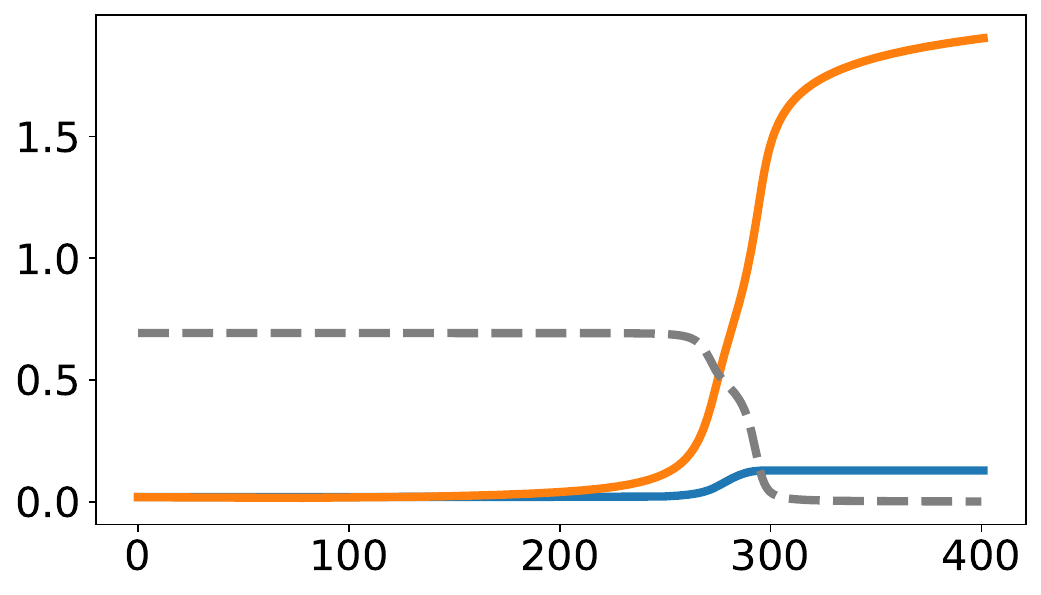}}
\subfigure[{\ourmethod}]{\label{fig:c}\includegraphics[width=0.3\textwidth]{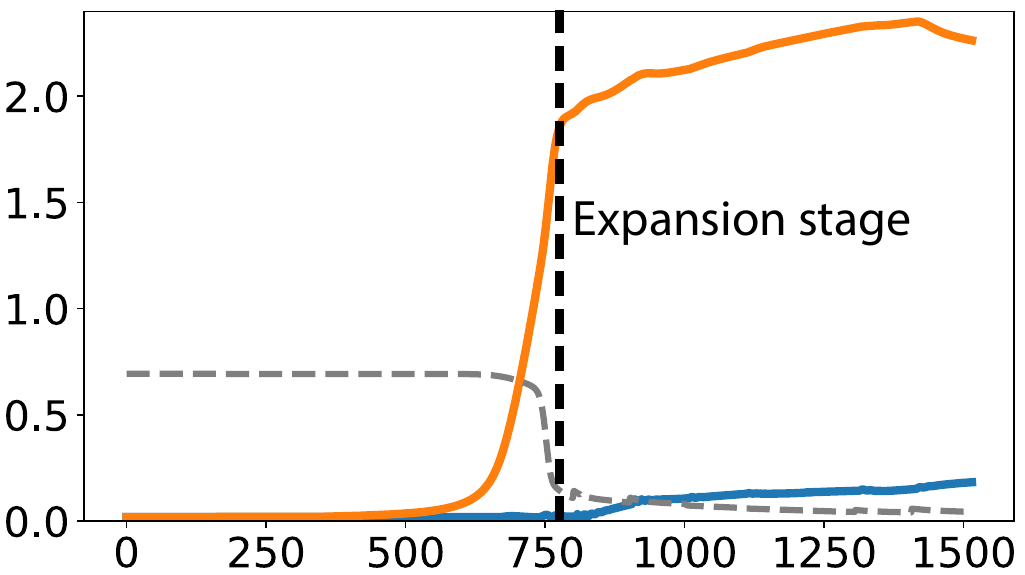}}
\caption{\textbf{Training process of ERM vs. {\ourmethod}.} We consider the same dataset generated from the distribution as in Definition~\ref{def:data_distribution} for ERM (case 1) and {\ourmethod}. On the same training data, ERM learns the spurious feature while {\ourmethod} successfully learns the core feature. We further consider ERM (case 2) when training on the data distribution where $\beta_c>\beta_s$ and $\hat{\alpha}=0.98$. 
We show the growth of the max inner product between the model's neuron and core/spurious signal vector and the decrease of training loss with regard to the number of iterations $t$.} 
\label{fig:demo_synthetic}
\end{figure}

\subsection{Real Data}
We conduct experiments on real benchmark datasets to (1) compare our approach with state-of-the-art methods, highlighting its superior performance and efficiency, and (2) offer insights into the design of our method through ablation studies. 

\noindent\textbf{Datasets.}
We evaluate on three wildly used datasets across vision and language tasks for spurious correlation: (1) \textbf{Waterbirds} \citep{sagawa2019distributionally} contains bird images labeled as waterbird or landbird, placed against a water or land background, where the smallest subgroup is waterbirds on land background. (2) \textbf{CelebA} \citep{liu2015deep} is used to study gender as the spurious feature for hair color classification, and the smallest group in this task is blond-haired males. (3) \textbf{CivilComments-WILDS} \citep{koh2021wilds} classifies toxic and non-toxic online comments while dealing with demographic information. It creates $16$ overlapping groups for each of the $8$ demographic identities. 

\noindent\textbf{Baselines.} We compare our proposed algorithm against several state-of-the-art methods. Apart from standard ERM, we include GroupDRO~\citep{sagawa2019distributionally} and DFR~\citep{kirichenko2023last} that assume access to the group labels. 
We note that $\text{DFR}^{\text{Val}}$ also uses the validation data for fine-tuning the last layer of the model. 
We also design a baseline called subsample that simply trains the model on the warm-up dataset only. Additionally, we evaluate three recent methods that address spurious correlations without the need for group labels: LfF~\citep{nam2020learning}, EIIL~\citep{creager2021environment}, and JTT~\citep{liu2021just}. 
We report results for ERM, Subsample, GroupDRO and {\ourmethod} based on our own runs using the WILDS library \cite{wilds2021}; for others, we directly reuse their reported numbers.

We present the experiment details including dataset statistics and hyperparameters as well as comprehensive additional experiments in Appendix~\ref{app:datasets}.\vspace{-.2cm}

\subsubsection{Consistent Superior Worst-group Performance}\vspace{-.2cm}
We assess {\ourmethod} on the mentioned datasets with state-of-the-art methods. Importantly, we emphasize the comparison with GroupDRO, as it represents the best-performing method that utilizes group information. As shown in Table~\ref{tab:baselines}, {\ourmethod} considerably enhances the worst-performing group's performance across all datasets, while maintaining the average accuracy comparable to GroupDRO. Considering all methods that only use validation data for model selection, GroupDRO still occasionally fails to surpass other methods. Remarkably {\ourmethod}'s performance consistently exceeds them in worst-group accuracy.\vspace{-.2cm}

\begin{table}[!t]
  \caption{The worst-group and average accuracy ($\%$) of {\ourmethod} compared with state-of-the-art methods. The \textbf{bold} numbers indicate the best results among the methods that \textit{require group information}, while the \underline{underscored} numbers represent methods that \textit{only train once}. All methods use validation data for early stopping and model selection, while $\surd$$\surd$ indicates that the method also re-trains the last layer using the validation data}
  \label{tab:baselines}
  \centering
  \resizebox{0.98\columnwidth}{!}{%
  \begin{tabular}{lcccararar}
    \toprule
    & \multirow{2}{2.5em}{\centering{\textbf{Group\\info}}} & \multirow{2}{2.5em}{\centering{\underline{Train}\\\underline{once}}} & \multirow{2}{2.5em}{\centering{Val\\info}} &\multicolumn{2}{c}{Waterbirds}&  \multicolumn{2}{c}{CelebA}&  \multicolumn{2}{c}{CivilComments}\\
    \cmidrule(r){5-10}
     Method & & & & Worst & Average & Worst & Average & Worst & Average\\
    \midrule
    ERM & $\times$ & $\surd$ & $\surd$ & $70.0_{\pm 2.3}$ & $97.1_{\pm 0.1}$ & $45.0_{\pm 1.5}$ & $94.8_{\pm 0.2}$ & $58.2_{\pm 2.8}$ & $92.2_{\pm 0.1}$ \\   
    LfF & $\times$ & $\times$ & $\surd$ & $78.0_{N/A}$ & $91.2_{N/A}$ & $77.2_{N/A}$ & $85.1_{N/A}$ & $58.8_{N/A}$ & $92.5_{N/A}$ \\
    EIIL & $\times$ & $\times$ & $\surd$ & $77.2_{\pm 1.0}$ & $96.5_{\pm 0.2}$ & $81.7_{\pm 0.8}$ & $85.7_{\pm 0.1}$ & $67.0_{\pm 2.4}$ & $90.5_{\pm 0.2}$ \\
    JTT & $\times$ & $\times$ & $\surd$ & $86.7_{N/A}$ & $93.3_{N/A}$ & $81.1_{N/A}$ & $88.0_{N/A}$ & $69.3_{N/A}$ & $91.1_{N/A}$ \\
    \midrule
    Subsample & $\surd$ & $\surd$ & $\surd$ & $86.9_{\pm 2.3}$ & $89.2_{\pm 1.2}$ & $86.1_{\pm 1.9}$ & $91.3_{\pm 0.2}$ & $64.7_{\pm 7.8}$ & $83.7_{\pm 3.4}$ \\ 
    $\text{DFR}^{\text{Tr}}$ & $\surd$ & $\times$ & $\surd$ & $90.2_{\pm 0.8}$ & $97.0_{\pm 0.3}$ & $80.7_{\pm 2.4}$ & $90.6_{\pm 0.7}$ & $58.0_{\pm 1.3}$ & $92.0_{\pm 0.1}$ \\
    $\text{DFR}^{\text{Val}}$ & $\surd$ & $\times$ & $\surd$$\surd$ & $\textbf{92.9}_{\pm 0.2}$ & $94.2_{\pm 0.4}$ & $88.3_{\pm 1.1}$ & $91.3_{\pm 0.3}$ & $70.1_{\pm 0.8}$ & $87.2_{\pm 0.3}$ \\
    GroupDRO & $\surd$ & $\surd$ & $\surd$ & $86.7_{\pm 0.6}$ & $93.2_{\pm 0.5}$ & $86.3_{\pm 1.1}$ & $92.9_{\pm 0.3}$ & $69.4_{\pm 0.9}$ & $89.6_{\pm 0.5}$ \\
    {\ourmethod} & $\surd$ & $\surd$ & $\surd$ & $\underline{90.3}_{\pm 0.3}$ & $92.4_{\pm 0.8}$ & $\underline{\textbf{91.0}}_{\pm
0.4}$ & $92.0_{\pm 0.6}$ & $\underline{\textbf{71.5}}_
{\pm 0.5}$ & $86.3_{\pm 1.7}$ \\
    \bottomrule
  \end{tabular}%
  }
\end{table}

\begin{table}[!t]
  \caption{Training efficiency of {\ourmethod} and GroupDRO on Waterbirds. We compare with GroupDRO at their learning rate and weight decay, as well as at ours. We report the worst-group accuracy, average accuracy and the number of epochs till early stopping as the model reached the best performance on validation data. Note: for a fair comparison, we consider one training epoch as training over the $N$ data as the size of the training dataset.}
  \label{tab:early-stop}
  \centering
  \begin{tabular}{lccrrc}
    \toprule
     Method & Learning rate & Weight decay & Worst & Average & Early-stopping epoch*\\
    \midrule
    GroupDRO & 1e-5 & 1e-0 & $86.7_{\pm 0.6}$ & $93.2_{\pm 0.5}$ & $92_{\pm 4}$ \\
    GroupDRO & 1e-2 & 1e-2 & \textcolor{maroon}{$77.3_{\pm 2.0}$} & $97.1_{\pm 0.5}$ & $15_{\pm 15}$ \\
    \ourmethod & 1e-2 & 1e-2 & $\textbf{90.3}_{\pm 0.3}$ & $92.4_{\pm 0.8}$ & $8.9_{\pm 1.8}$ \\
    \bottomrule
  \end{tabular}\vspace{-.5cm}
\end{table}

\subsubsection{Efficient Training}\vspace{-.2cm}
In this subsection, we show that our method is more efficient as it does not train a model twice (as in JTT) and more importantly avoids the necessity for a small learning rate (as in GroupDRO). 
Specifically, methods employing group-balanced batches like GroupDRO require a very small learning rate coupled with a large weight decay in practice. 
We provide an intuitive explanation as follows.
When sampling to achieve balanced groups in each batch, smaller groups appear more frequently than larger ones. 
If training progresses rapidly, the loss on smaller groups will be minimized quickly, while the majority of the large group data remains unseen and contributes to most of the gradients in later batches. 
Therefore, these methods necessitate slow training to ensure the model encounters diverse data from larger groups before completely learning the smaller groups. 
We validate this observation in Table~\ref{tab:early-stop}, where GroupDRO trained faster than the default results in significantly poorer performance similar to ERM.
Conversely, {\ourmethod} can be trained to converge rapidly on the warm-up set and reaches better worst-group accuracy $10\times$ faster than GroupDRO at default.  
Note that methods which only finetune the last layer~\citep{kirichenko2023last,wei2023distributionally} are also efficient. However, they still require training a model first using ERM on the entire training data till convergence. In contrast, PDE does not require further finetuning of the model.

\subsubsection{Understanding the Two Stages}
\begin{wraptable}{r}{6.8cm}
\caption{Performance of {\ourmethod} after each stage. We report the worst-group and average accuracy.}\label{tab:stage}
    \begin{tabular}{lccrrrrrr}
    \toprule
    & \multicolumn{2}{c}{Warm-up} & \multicolumn{2}{c}{Addition} \\
     Dataset & Worst & Avg  & Worst & Avg \\
    \midrule
    Waterbirds & 86.0 & 91.9 & \textbf{90.3} & 92.4 \\
    CelebA & 87.8 & 92.1 & \textbf{91.0} & 92.0 \\
    CivilComm & 67.7 & 78.8 & \textbf{71.5} & 86.3 \\
    \bottomrule
    \end{tabular}
\end{wraptable} 
We examine each component and demonstrate their effect in {\ourmethod}. In Table~\ref{tab:stage}, we present the worst-group and average accuracy of the model trained following the warm-up and expansion stages. Indeed, the majority of learning occurs in the warm-up stage, during which a satisfactory worst-group accuracy is established. 
In the expansion stage, the model persists in learning new data along the established trajectory, leading to continued performance improvement. In Figure~\ref{fig:reinit_momentum}, we corroborate and emphasize that the model has acquired the appropriate features and maintains learning based on its historical gradient stored in the momentum. 
As shown, if the optimizer is reset after the warm-up stage and loses all its historical gradients (with reinitialization), it soon acquires spurious features, resulting in a swift decline in performance accuracy as shown in the blue line.

\begin{figure}[!htp]
\centering   
\subfigure[Worst-group accuracy.]{\label{fig:wg_acc}\includegraphics[width=0.45\textwidth]{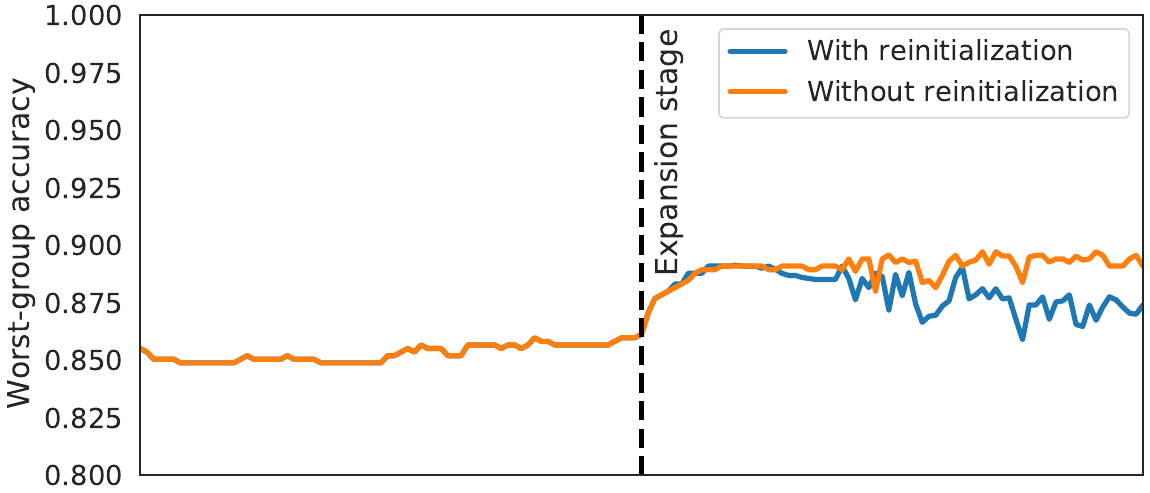}}
\subfigure[Average accuracy.]{\label{fig:avg_acc}\includegraphics[width=0.44\textwidth]{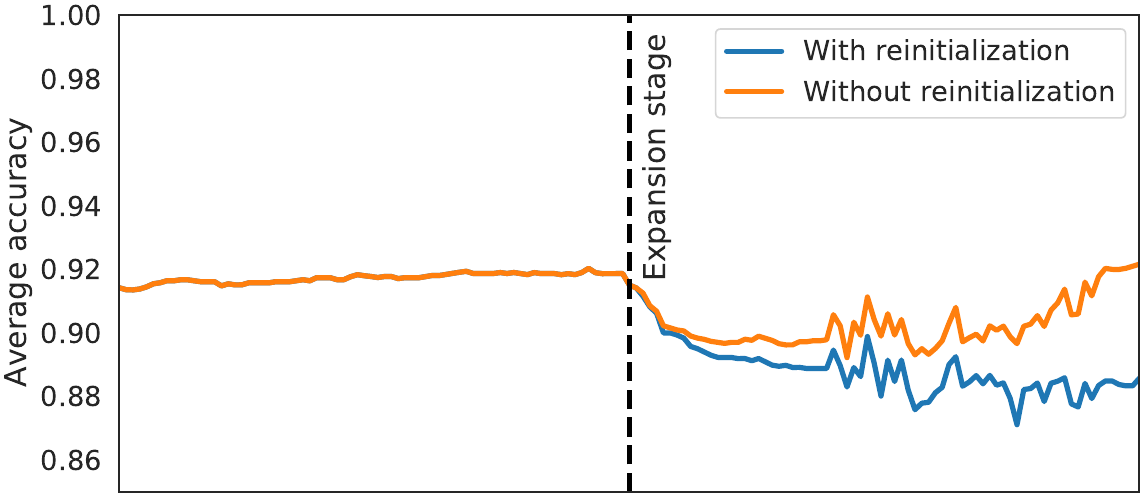}}
\caption{The effect of resetting the momentum after the warm-up stage for {\ourmethod} on Waterbirds.}
\label{fig:reinit_momentum}
\end{figure}

\subsubsection{Ablation Study on the Hyper-parameters of {\ourmethod}}
{\ourmethod} is robust within a reasonable range of hyperparameter choices, although some configurations outperform others. 
As shown in Table~\ref{tab:ablation_size}, it is necessary to limit the number of data points introduced during each expansion to prevent performance degradation.
Similarly, in Appendix~\ref{appendix:synthetic}, we emphasize the importance of gradual data expansion. 
In Table~\ref{tab:ablation_lr}, we show that post-warmup learning rate decay is essential, though PDE exhibits tolerance to the degree of this decay. 
Lastly, as illustrated in Figure~\ref{fig:reinit_momentum}, adopting a smaller learning rate often necessitates increased data expansions. 
Nonetheless, a reduced learning rate does not necessarily lead to improved performance.

\begin{figure}[!htp]
\centering     
\subfigure[Worst-group accuracy.]{\label{fig:wg_acc}\includegraphics[width=0.47\textwidth]{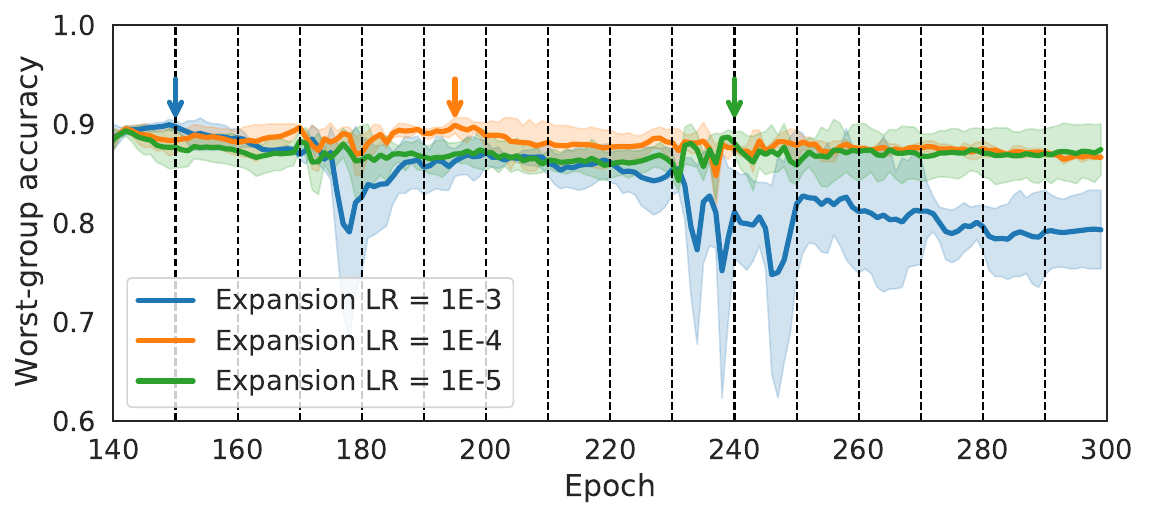}}
\subfigure[Average accuracy.]{\label{fig:avg_acc}\includegraphics[width=0.47\textwidth]{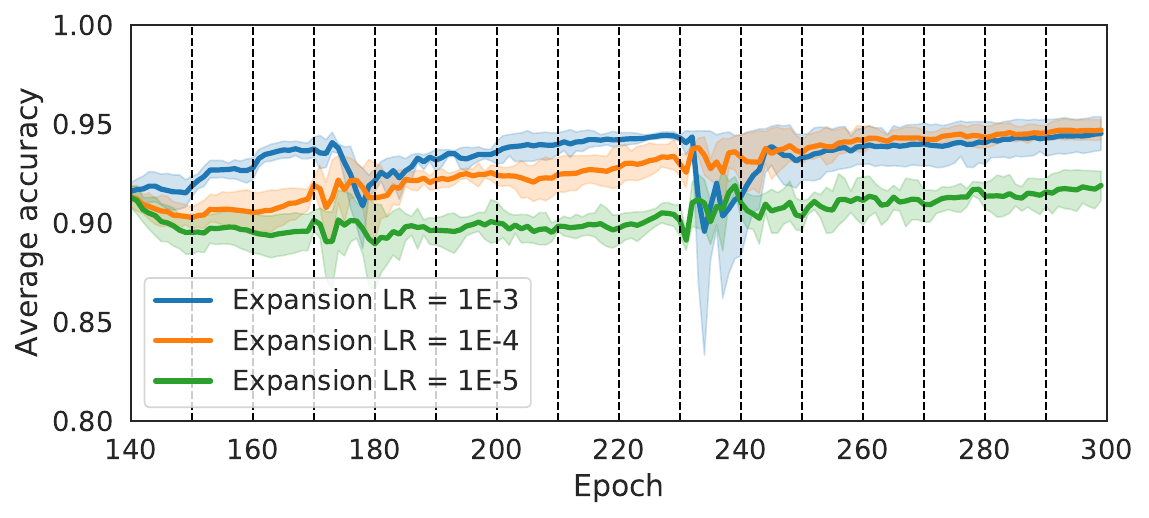}}
\caption{The variations in both worst-group and average accuracy on the test set of Waterbirds during the expansion stage under different expansion learning rates. Each vertical dashed line denotes an expansion and the arrow denotes the early stopping.}
\label{fig:reinit_momentum}
\end{figure}

\begin{minipage}{\textwidth}
    \begin{minipage}[b]{0.48\textwidth}
    \captionof{table}{Ablation study on Waterbirds. Exp. size: number of data points added in each expansion.}
    \centering
    \begin{tabular}{lrrr}
    \toprule
    Exp. size & Exp. lr & Worst & Average \\
    \midrule
    5 & 1e-4 & $89.9_{\pm 0.5}$ & $92.1_{\pm 0.3}$\\
    10 & 1e-4 & $\textbf{90.3}_{\pm 0.3}$ & $92.4_{\pm 0.8}$\\
    50 & 1e-4 & $88.1_{\pm 0.8}$ & $93.4_{\pm 0.4}$\\
    \bottomrule
  \end{tabular}
    \label{tab:ablation_size}
    \end{minipage}
      \hfill
  \begin{minipage}[b]{0.48\textwidth}
  \captionof{table}{Ablation study on Waterbirds. Exp. lr: the learning rate in the expansion stage.}
    \centering
    \begin{tabular}{lrrr}
    \toprule
    Exp. size & Exp. lr & Worst & Average \\
    \midrule
    10 & 1e-2 & $\textcolor{red}{85.4}_{\pm 3.1}$ & $92.1_{\pm 2.0}$ \\
    10 & 1e-3 & $89.4_{\pm 0.7}$ & $92.6_{\pm 0.3}$\\
    10 & 1e-5 & $89.5_{\pm 0.2}$ & $92.1_{\pm 0.1}$\\
    \bottomrule
  \end{tabular}
    \label{tab:ablation_lr}
  \end{minipage}
\end{minipage}

\section{Related Work}
Existing approaches for improving robustness against spurious correlations can be categorized into two lines of research based on the tackled subproblems. A line of research focuses on the same subproblem we tackle: effectively using the group information to improve robustness. 
With group information, one can use the distributionally robust optimization (DRO) framework and dynamically increase the weight of the worst-group loss in minimization~\citep{hu2018does,oren-etal-2019-distributionally,sagawa2019distributionally,zhang2021coping}. 
Within this line of work, GroupDRO~\citep{sagawa2019distributionally} achieves state-of-the-art performances across multiple benchmarks. 
Other approaches use importance weighting to reweight the groups~\citep{shimodaira2000improving,byrd2019effect,xu2021understanding} and class balancing to downsample the majority or upsample the minority~\citep{he2009learning,cui2019class,sagawa2020investigation}. 
Alternatively, \citet{goel2021model} leverage group information to augment the minority groups with synthetic examples generated using GAN. 
Another strategy~\citep{cao2019learning,cao2020heteroskedastic} involves imposing Lipschitz regularization around minority data points. Most recently, methods that train a model using ERM first and then only finetune the last layer on balanced data from training or validation~\citep{kirichenko2023last}, or on mixed representations~\citep{xue2023eliminating}, or learn post-doc scaling adjustments~\citep{wei2023distributionally} are shown to be effective. 

The other line of research focuses on the setting where group information is not available during training and tackles the first subproblem we identified as accurately finding the groups. 
Recent notable works~\citep{nam2020learning,liu2021just,creager2021environment,zhang2022correct,yang2023identifying} mostly involve training two models, one of which is used to find group information.
To finally use the found groups, many approaches~\citep{namkoong2017variance,duchi2019distributionally,oren-etal-2019-distributionally,sohoni2020no} still follow the DRO framework. 

The first theoretical analysis of spurious correlation is provided by \citet{sagawa2020investigation}.
For self-supervised learning, \citet{chen2020self} shows that fine-tuning with pre-trained models can reduce the harmful effects of spurious features. \citet{ye2022freeze} provides guarantees in the presence of label noise that core features are learned well only when less noisy than spurious features.
These theoretical works only provide analyses of linear models. 
Meanwhile, a parallel line of work has established theoretical analysis of nonlinear CNNs in the more realistic setting~\citet{allen2020towards,zou2021understanding,wen2021toward,chen2022towards,jelassi2022towards}. Our work builds on this line of research and generalizes it to the study of spurious features. Lastly, we notice that a concurrent work~\citep{chen2023towards} also uses tensor power method~\citep{allen2020towards} to analyze the learning of spurious features v.s. invariant features, but in the setting of out-of-distribution generalization. 

\section{Conclusion}
In conclusion, this paper addressed the challenge of spurious correlations in training deep learning models and focused on the most effective use of group information to improve robustness. We provided a theoretical analysis based on a simplified data model and a two-layer nonlinear CNN. Building upon this understanding, we proposed {\ourmethod}, a novel training algorithm that effectively and efficiently enhances model robustness against spurious correlations. This work contributes to both the theoretical understanding and practical application of mitigating spurious correlations, paving the way for more reliable and robust deep learning models. 

\noindent\textbf{Limitations and future work.} Although beyond the linear setting, our analysis still focuses on a relatively simplified binary classification data model. To better represent real-world application scenarios, future work could involve extending to multi-class classification problems and examining the training of transformer architectures. Practically, our proposed method requires the tuning of additional hyperparameters, including the number of warm-up epochs, the number of times for dataset expansion and the number of data to be added in each expansion. 

\section*{Acknowledgements}

We sincerely thank Dongruo Zhou for the constructive suggestions on the structure and writing of the paper. We also thank the anonymous reviewers for their helpful comments. YD and QG are supported in part by the National Science Foundation CAREER Award 1906169 and IIS-2008981, and the Sloan Research Fellowship. BM is supported by the National Science Foundation CAREER Award 2146492. The views and conclusions contained in this paper are those of the authors and should not be interpreted as representing any funding agencies.

\bibliography{neurips_2023_spurious}
\bibliographystyle{icml2023}

\newpage
\appendix

\section{Synthetic Experiments}\label{appendix:synthetic}
\textbf{Datasets.} We generate $10,000$ training examples and $10,000$ test examples from the data distribution defined in Definition~\ref{def:data_distribution} with dimension $d = 50$ and number of patches $P=3$. Specifically, we let $\alpha=0.98$, $\beta_c=0.2$, $\beta_s=1$ and $\sigma_p=0.78$ for Table~\ref{tab:synthetic} as well as Figure~\ref{fig:a} and Figure~\ref{fig:c}. For Figure~\ref{fig:b}, we consider a data distribution where $\alpha=0.98$, $\beta_c=1$, $\beta_s=0.2$ and $\sigma_p=0.78$. Furthermore, we randomly shuffle the order of the patches of $\xb$ after we generate data $(\xb, y, a)$.  

\textbf{Training.} We consider the performances of a nonlinear CNN trained with ERM and {\ourmethod}. The nonlinear CNN architecture follows \eqref{eq:model} with the cubic activation function, where we let the number of neurons/filters $J=40$. We use gradient descent with momentum (GD+M) as the optimizer of our method, setting the momentum to $0.9$ and the learning rate to $0.03$. The number of warm-up iterations is set to $800$. 
We consider ERM trained with GD with a learning rate $0.1$ and without momentum to align with our theoretical finding in both Table~\ref{tab:synthetic} and Figure~\ref{fig:demo_synthetic}. 
In Table~\ref{tab:synthetic}, we also show the experiment results for ERM trained with GD+M as same as {\ourmethod}.
All models are trained until convergence.

\textbf{Additional experiments.} In Figure~\ref{fig:demo_synthetic_gdm}, we demonstrate the growth of $\max_{j\in[J]}\la\wb_j^{(t)},\vb_s\ra$ and $\max_{j\in[J]}\la\wb_j^{(t)},\vb_c\ra$ for ERM trained with GD+M under the same data generated in Figure~\ref{fig:demo_synthetic}. Similarly, we observe that ERM learns the spurious feature quickly as the training loss is minimized under our data distribution. Meanwhile, if the data is generated as in case 2 where $\beta_c>\beta_s$, ERM learns the core feature correctly.
\begin{figure}[htp]
\centering     
\subfigure[ERM (case 1)]{\label{fig:erm_gdm_a}\includegraphics[width=0.3\textwidth]{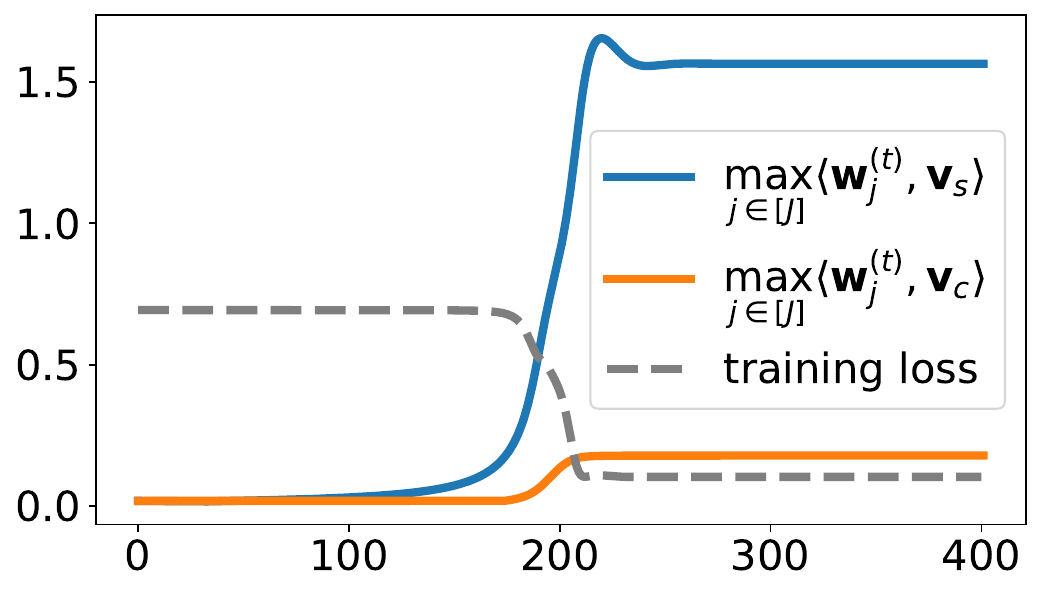}}
\subfigure[ERM (case 2)]
{\label{fig:erm_gdm_b}\includegraphics[width=0.3\textwidth]{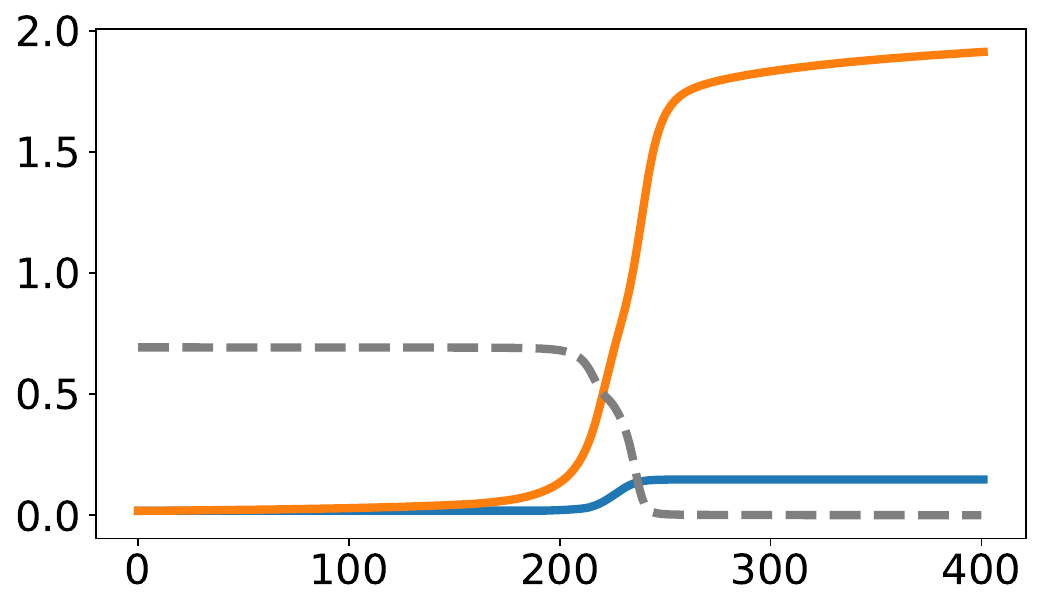}}
\caption{\textbf{Training process of ERM trained with GD+M.} We consider the same dataset generated in Figure~\ref{fig:demo_synthetic} and observe almost the same training process as ERM with GD, except GD+M learns the features faster.} 
\label{fig:demo_synthetic_gdm}
\end{figure}

Furthermore, we consider the following variation of our methods on the same dataset in Table~\ref{tab:synthetic} to demonstrate the importance of gradual expansion. In Figure~\ref{fig:demo_synthetic_ablation}, we let PDE incorporate all of the new training data at once after the warm-up stage. As demonstrated, adding all data at once makes it harder for the model to continue learning core features, resulting in a worst-group accuracy of $74.24\%$ as compared to $94.32\%$ for progressive expansion.
\begin{figure}[htp]
\centering     
\includegraphics[width=0.33\textwidth]{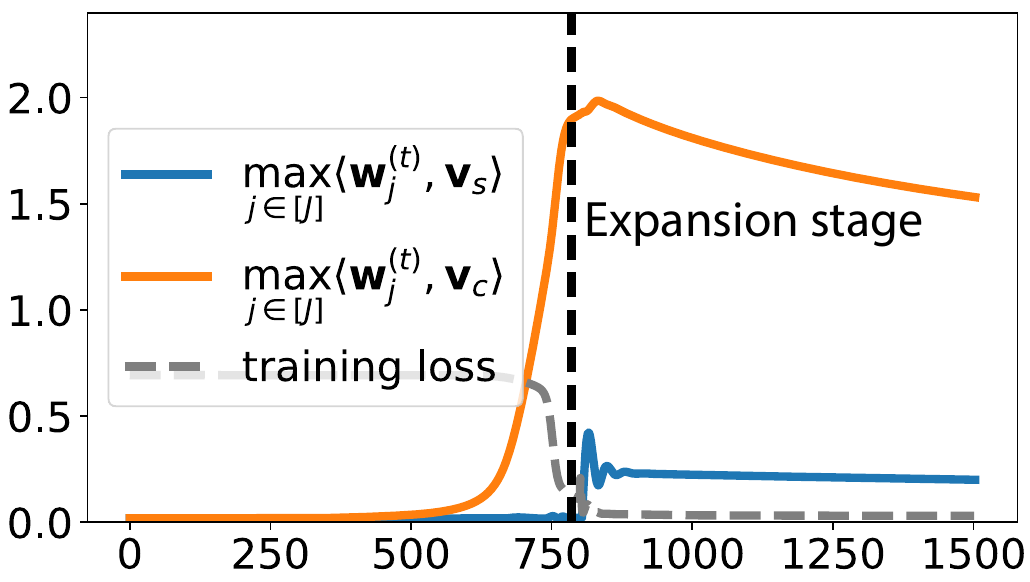}
\caption{\textbf{Variation of {\ourmethod}.} We consider the same dataset generated in Figure~\ref{fig:demo_synthetic} and add all data at once after the warm-up stage.} 
\label{fig:demo_synthetic_ablation}
\end{figure}

\section{Benchmark Datasets}
\label{app:datasets}
\textbf{Waterbirds.} The Waterbirds dataset \citep{sagawa2019distributionally} was constructed to study object recognition models relying on image backgrounds instead of the object itself. To this end, bird images from the Caltech-UCSD Birds-200-2011 (CUB) dataset \citep{wah2011caltech} were combined with backgrounds from the Places dataset \citep{zhou2017places}. 
The dataset contains $4,795$ bird images labeled as a waterbird or landbird and placed against a water or land background. Waterbirds are predominantly located against a water background, while landbirds are situated against a land background. Notably, the smallest subgroup in the dataset is waterbirds on land, consisting of only $56$ examples.

\textbf{CelebA.} The CelebA dataset \citep{liu2015deep} is a popular face attribute dataset used to examine the spurious associations between non-demographic and demographic attributes. Specifically, one of the $40$ binary attributes, ``blond hair", is used as the target attribute, and ``male" is the spurious attribute. The dataset contains $162,770$ training examples, with the smallest group being blond-haired males, with only 1387 examples. 

\textbf{CivilComments-WILDS.} The CivilComments-WILDS dataset \citep{koh2021wilds} is designed to explore the challenge of classifying online comments as either toxic or non-toxic while dealing with the spurious correlation between the label and demographic information such as gender, race, religion, and sexual orientation. The dataset's evaluation metric, as defined by \cite{koh2021wilds}, creates 16 overlapping groups for each of the eight demographic identities, resulting in a total of $512$ distinct groups. For each group, the metric calculates the worst-case performance of a classifier, which allows for a robust evaluation of the model's ability to generalize across diverse populations.

\section{Real Data Experiments}
\textbf{Setup.} Our experiment settings strictly follow the same setting used for datasets introduced in \cref{app:datasets} in previous works \citep{sagawa2019distributionally,liu2021just,nam2020learning,creager2021environment,kirichenko2023last}. Specifically, we built our training pipeline with the WILDS package \citep{wilds2021} which uses pretrained ResNet-50 model \citep{he2016deep} in Pytorch \citep{pytorch} library for the image datasets (i.e., Waterbirds and CelebA) and Transformer \citep{vaswani2017attention} in Transformers library \citep{wolf-etal-2020-transformers} for CivilComments-WILDS.  All experiments were conducted on a single NVIDIA RTX A6000 GPU with 48GB memory. 

\begin{table}
  \caption{Number of data in our warm-up dataset for \ourmethod's results in \cref{tab:baselines}. We also report the number of data in total for the three datasets.}
  \label{tab:app-datanum}
  \centering
  \begin{tabular}{lcc}
    \toprule
     Dataset & Warm-up & All \\
    \midrule
    Waterbirds & 224 & 4,795\\
    CelebA & 5,548 & 162,770 \\
    CivilComments-WILDS & 13,705 & 269,038 \\
    \bottomrule
  \end{tabular}
\end{table}

\textbf{Training.} In Table~\ref{tab:app-datanum}, we summarize the number of data used in the warm-up stage for {\ourmethod} in  \cref{tab:baselines} with the total number of data in the entire datasets. In \cref{tab:app-hyper}, we report the hyperparameters used for {\ourmethod} with the notations in \cref{alg:method}. 
Specifically, $T_0$ refers to the number of epochs for the warm-up stage and $J$ refers to the number of epochs for training after each data expansion. Lastly, $m$ is the number of added data for each data expansion. Our batch size is consistent with GroupDRO.
\begin{table}[ht]
  \caption{Hyperparameters used for \ourmethod's results in \cref{tab:baselines}. Note that $T_0$ and $J$ are in epochs of \ourmethod's training set, which have fewer iterations than epochs of the full training set.}
  \label{tab:app-hyper}
  \centering
  \begin{tabular}{lccccccc}
    \toprule
     Dataset & Learning rate & Weight decay & Batch size & $T_0$ & $J$ & $m$ \\
    \midrule
    Waterbirds & 1e-2 & 1e-2 & 64 & 140 & 10 & 10  \\
    CelebA & 1e-2 & 1e-4 & 128 & 16 & 10 & 50 \\
    CivilComments-WILDS & 1e-5 & 1e-2 & 16 & 15 & 2 & 300 \\
    \bottomrule
  \end{tabular}
\end{table}

\textbf{Groups for CivilComments-WILDS.} We note that the demographic tags in CivilComments-WILDS can coexist in the input text. For example, a text can contain both tags of female and male. Therefore, combining the 8 demographic tags with the binary classification label (toxic vs. non-toxic) results in 16 overlapping groups, where each group counts as data from a class with/without a specific tag. For computational efficiency, previous methods divide the data into four non-overlapping groups either by the \textit{specific} one demographic tag $a_i$ (groups are $\{y=\pm1,a_i=\pm1\}$) \citep{koh2021wilds} or by containing \textit{any} one of the tags: $a=1$ if any $a_i=1$ and $a=-1$ otherwise (groups are $\{y=\pm1,a=\pm1\}$) \citep{liu2021just,creager2021environment}. However, the data can actually be partitioned into $512$ distinct groups, with each group corresponding to different combinations of tags: $\{y=\pm1,a_1=\pm1,a_2=\pm1,\ldots,a_n=\pm1\}$. As GroupDRO requires computation per group at each training batch, considering a large number of groups makes it harder for GroupDRO to train efficiently. Meanwhile, having more groups does not impose an additional computational cost on \ourmethod, so we can consider all these data groups when constructing our warm-up set. As many groups are empty or contain very little data, we set a threshold to select at most $150$ data points from each group to ensure a balanced yet sufficient warm-up set.

\textbf{Efficiency.}  In Table~\ref{tab:early-stop-appendix}, we further report the training efficiency of {\ourmethod} compared with GroupDRO on CelebA and CivilComments-WILDS. Similar to what we observe on the Waterbirds dataset, {\ourmethod} achieves the best performance at a larger learning rate and smaller weight decay on CelebA with a significant speedup as compared to GroupDRO. On CivilComments-WILDS, we can also observe an improved efficiency.
\begin{table}[ht]
  \caption{Training efficiency of {\ourmethod} and GroupDRO on CelebA dataset.}
  \label{tab:early-stop-appendix}
  \centering
  \begin{tabular}{lccrrc}
    \toprule
     Method & Learning rate & Weight decay & Worst & Average & Early-stopping epoch*\\
    \midrule
    GroupDRO & 1e-5 & 1e-1 & $86.3_{\pm 1.1}$ & $92.9_{\pm 0.3}$ & $23.7_{\pm 6.8}$ \\
    \ourmethod & 1e-2 & 1e-4 & $\textbf{91.0}_{\pm
 0.4}$ & $92.0_{\pm 0.6}$ & $\textbf{0.7}_{\pm 0.3}$ \\
    \bottomrule
  \end{tabular}
\end{table}

\begin{table}[ht]
  \caption{Training efficiency of {\ourmethod} and GroupDRO on CivilComments-WILDS dataset.}
  \label{tab:early-stop-appendix-civil}
  \centering
  \begin{tabular}{lccrrc}
    \toprule
     Method & Learning rate & Weight decay & Worst & Average & Early-stopping epoch*\\
    \midrule
    GroupDRO & 1e-5 & 1e-2  & $69.4_{\pm 0.9}$ & $89.6_{\pm 0.5}$ & $3.3_{\pm 2.1}$\\
    \ourmethod & 1e-5 & 1e-2 & $\textbf{71.5}_
 {\pm 0.5}$ & $86.3_{\pm 1.7}$ & $\textbf{2.1}_{\pm 1.1}$ \\
    \bottomrule
  \end{tabular}
\end{table}

\noindent\textbf{Data Augmentation.} Additionally, the increased training speed of our method facilitates the usage of techniques such as data augmentation. 
While data augmentation is a common practice for improving model generalization, DRO approaches have not incorporated it into their methods. 
We hypothesize that this omission stems from the slower training process. Data augmentation introduces random noise to the training data, which complicates convergence during training when using a very small learning rate. As illustrated in Table~\ref{tab:aug}, data augmentation leads to slightly worse performance for GroupDRO. In contrast, our method effectively benefits from data augmentation.
\begin{table}[ht]
    \captionof{table}{The effect of data augmentation on GroupDRO and {\ourmethod} on Waterbirds dataset. We report the worst-group and average accuracy.}
    \centering
    \begin{tabular}{lccrrrrrr}
    \toprule
    & \multicolumn{2}{c}{GroupDRO} & \multicolumn{2}{c}{{\ourmethod}} \\
     Method & Worst & Avg & Worst & Avg \\
    \midrule
    W/o data aug & $86.7$ & $93.2$ & $88.9$ & $89.5$ \\
    W/ data aug & \textcolor{maroon}{$85.7$} & $96.6$ & $\textbf{90.3}$ & $\textbf{92.4}$ \\
    \bottomrule
    \end{tabular}
    \label{tab:aug}
\end{table}

\vspace{3pt}
\section{Proof Preliminaries}
\noindent\textbf{Notation.} In this paper, we use lowercase letters, lowercase boldface letters, and uppercase boldface letters to respectively denote scalars ($a$), vectors ($\vb$), and matrices ($\Wb$). We use $\text{sgn}$ to denote the sign function.For a vector $\vb$, we use $\|\vb\|_2$ to denote its Euclidean norm. Given two sequences $\{x_n\}$ and $\{y_n\}$, we denote $x_n = \cO(y_n)$ if $|x_n|\le C_1 |y_n|$ for some absolute positive constant $C_1$, $x_n = \Omega(y_n)$ if $|x_n|\ge C_2 |y_n|$ for some absolute positive constant $C_2$, and $x_n = \Theta(y_n)$ if $C_3|y_n|\le|x_n|\le C_4 |y_n|$ for some absolute constants $C_3,C_4>0$. We use $\tilde \cO(\cdot)$ to hide logarithmic factors of $d$ in $\cO(\cdot)$.

Before we go into the analysis, we first consider the following gradient,
\begin{align}
    \nabla_{\wb_j}\cL(\Wb^{(t)})=-\frac{1}{N}\sum^{N}_{i=1}\frac{\exp(-y_if(\xb_i;\Wb^{(t)}))}{1+\exp(-y_if(\xb_i;\Wb^{(t)}))}\cdot y_i f'(\xb_i;\Wb^{(t)}).
\end{align}
Let's denote the derivative of a data example $i$ at iteration $t$ to be 
\begin{align}
\ell^{(t)}_i=\frac{\exp(-y_if(\xb_i;\Wb^{(t)}))}{1+\exp(-y_if(\xb_i;\Wb^{(t)}))}=\text{sigmoid}(-y_if(\xb_i;\Wb^{(t)})).
\end{align} 
\begin{lemma}\label{lemma:grad}
    (Gradient) Let the loss function $\cL$ be as defined in \eqref{eq:empirical loss}. For $t\ge 0$ and $j\in [J]$, the gradient of the loss $\cL$ with regard to neuron $\wb_j$ is
    \begin{align*}
        \nabla_{\wb_j}\cL(\Wb^{(t)})&=-\frac{3}{N}\bigg(\beta_c^3\sum_{i=1}^N \ell_i^{(t)}\la\wb_j,\vb_c\ra^2\vb_c +\sum_{i=1}^N\ell_i^{(t)}y_i\la\wb_j,\bxi_i\ra^2\bxi_i+\\
        &\qquad\Big(\sum_{i\in S_1}\ell_i^{(t)}-\sum_{i\in S_2}\ell_i^{(t)}\Big)\cdot\beta_s^3\la\wb_j,\vb_s\ra^2\vb_s\bigg).
    \end{align*}
\end{lemma}
\begin{proof}
We have the following gradient
\begin{align}
    \nabla_{\wb_j}\cL(\Wb^{(t)})=-\frac{1}{N}\sum^{N}_{i=1}\frac{\exp(-y_if(\xb_i;\Wb^{(t)}))}{1+\exp(-y_if(\xb_i;\Wb^{(t)}))}\cdot y_i f'(\xb_i;\Wb^{(t)}).
\end{align}
And let's denote the derivative of a data example $i$ at iteration $t$ to be 
\begin{align}
\ell^{(t)}_i=\frac{\exp(-y_if(\xb_i;\Wb^{(t)}))}{1+\exp(-y_if(\xb_i;\Wb^{(t)}))}=\text{sigmoid}(-y_if(\xb_i;\Wb^{(t)})).
\end{align}
Then, we can further write the gradient as 
\begin{align*}
    \nabla_{\wb_j}\cL(\Wb^{(t)})&=-\frac{3}{N}\sum^{N}_{i=1}\ell^{(t)}_iy_i\sum^P_{p=1}\la\wb_j,\xb^{(p)}\ra^2\cdot\xb^{(p)}\\
    &= -\frac{3}{N}\sum^{N}_{i=1}\ell^{(t)}_iy_i \Big(\la\wb_j,\beta_cy_i\vb_c\ra^2\beta_cy_i\vb_c+\la\wb_j,\beta_sa_i\vb_s\ra^2\beta_sa_i\vb_s+\la\wb_j,\bxi_i\ra^2\bxi_i\Big) \\
    &= -\frac{3}{N}\sum^{N}_{i=1}\ell^{(t)}_i \Big(\beta_c^3\la\wb_j,\vb_c\ra^2\vb_c+\beta_s^3y_ia_i\la\wb_j,\vb_s\ra^2\vb_s+y_i\la\wb_j,\bxi_i\ra^2\bxi_i\Big)\\
    &=-\frac{3}{N}\Bigg(\sum_{i=1}^N\ell_i^{(t)}\bigg(\beta_c^3\la\wb_j,\vb_c\ra^2\vb_c+y_i\la\wb_j,\bxi_i\ra^2\bxi_i\bigg)\\
    &\qquad+\Big(\sum_{i\in S_1}\ell_i^{(t)}-\sum_{i\in S_2}\ell_i^{(t)}\Big) \beta_s^3\la\wb_j,\vb_s\ra^2\vb_s\Bigg), 
\end{align*}
where the last equality holds due to that for $i\in S_1$ we have $a_i=y_i$ and for $i\in S_2$ we have $a_i=-y_i$.
\end{proof}
With the gradient, we have the following:

\noindent\textbf{Core feature gradient.} The projection of the gradient on $\vb_c$ is then 
\begin{align}
    \la\nabla_{\wb_j}\cL(\Wb^{(t)}),\vb_c\ra=-\frac{3\beta_c^3}{N}\sum_{i=1}^N \ell_i^{(t)}\la\wb_j,\vb_c\ra^2.
     \label{eq:v_c grad update}
\end{align}
\noindent\textbf{Spurious feature gradient.} The projection of the gradient on $\vb_s$ is
\begin{align}
    \la\nabla_{\wb_j}\cL(\Wb^{(t)}),\vb_s\ra=-\frac{3\beta_s^3}{N}\Big(\sum_{i\in S_1}\ell_i^{(t)}-\sum_{i\in S_2}\ell_i^{(t)}\Big)\cdot\la\wb_j,\vb_s\ra^2.
    \label{eq:v_s grad update}
\end{align}
\noindent\textbf{Noise gradient.} The projection of the gradient on $\bxi_i$ is
\begin{align}
    \la\nabla_{\wb_j}\cL(\Wb^{(t)}),\bxi_i\ra=-\frac{3y_i}{N}\sum_{i=1}^N \ell_i^{(t)}\la\wb_j,\bxi_i\ra^2 \|\bxi_i\|_2^2.
    \label{eq:noise grad update}
\end{align}
\noindent\textbf{Derivative of data example $i$.} $\ell^{(t)}_i$ can be rewritten as 
\begin{align}
    \ell^{(t)}_i &= \text{sigmoid}\big(-y_if(\xb_i;\Wb^{(t)})\big)\nonumber\\
    &= \text{sigmoid}\Big(\sum^J_{j=1}-\beta_c^3\la\wb_j,\vb_c\ra^3-y_ia_i\beta_s^3\la\wb_j,\vb_s\ra^3-y_i\la\wb_j,\bxi_i\ra^3\Big).\label{eq:data derivative}
\end{align}
Note that $0<\ell^{(t)}_i<1$ due to the property of the sigmoid function. Furthermore, we similarly consider that the sum of the sigmoid terms for all time steps is bounded up to a logarithmic dependence~\citep{chen2022towards}. The sigmoid term is considered small for a $\kappa$ such that
\begin{align*}
    \sum_{t=0}^{T}\frac{1}{1+\exp(\kappa)} \le \tilde{O}(1),
\end{align*}
which implies $\kappa\ge\tilde{\Omega}(1)$. 

\section{Proof of Theorem~\ref{lemma:stage1 spurous}}\label{appendix:analysis}
In this section, we present the detailed proofs that build up to Theorem~\ref{lemma:stage1 spurous}. We begin by considering the update for the spurious feature and core feature.
\begin{lemma}[Spurious feature update.]\label{lemma:spurious update all}
    For all $t\ge 0$ and $j\in [J]$, the spurious feature update is
    \[\la\wb_j^{(t+1)},\vb_s\ra=\la\wb_j^{(t)},\vb_s\ra+\frac{3\eta\beta_s^3}{N}\Big(\sum_{i\in S_1}\ell_i^{(t)}-\sum_{i\in S_2}\ell_i^{(t)}\Big)\la\wb_j^{(t)},\vb_s\ra^2,\]
    which gives
    \begin{align*}
       \tilde\Theta(\eta)\beta_s^3\Big(\hat{\alpha} g_1(t) - \sum_{i\in S_2}\ell_i^{(t)}/N \Big)\la\wb_j^{(t)},\vb_s\ra^2
       &\le\la\wb_j^{(t+1)},\vb_s\ra- \la\wb_j^{(t)},\vb_s\ra\\
       &\le \tilde\Theta(\eta)\beta_s^3\cdot\hat{\alpha} g_1(t) \cdot \la\wb_j^{(t)},\vb_s\ra^2,
    \end{align*}
    where $g_1(t)=\text{sigmoid}\big(-\sum_{j\in[J]}(\beta_c^3 \la\wb_j^{(t)},\vb_c\ra^3+\beta_s^3 \la\wb_j^{(t)},\vb_s\ra^3)\big)$.
\end{lemma}
\begin{proof}
    The spurious feature update is obtained by using the gradient update of $\Wb^{(t)}$ and plugging in \eqref{eq:v_s grad update}:
    \begin{align*}
        \la\wb_j^{(t+1)},\vb_s\ra
        &=\la\wb_j^{(t)}-\eta \nabla_{\wb_j}\cL(\Wb^{(t)}),\vb_s\ra\\
        &=\la\wb_j^{(t)},\vb_s\ra+\frac{3\eta\beta_s^3}{N}\Big(\sum_{i\in S_1}\ell_i^{(t)}-\sum_{i\in S_2}\ell_i^{(t)}\Big)\la\wb_j^{(t)},\vb_s\ra^2.
    \end{align*}
    We first prove the upper bound. Consider the following,
    \begin{align*}
        \la\wb_j^{(t+1)},\vb_s\ra        &=\la\wb_j^{(t)},\vb_s\ra+\frac{3\eta\beta_s^3}{N}\Big(\sum_{i\in S_1}\ell_i^{(t)}-\sum_{i\in S_2}\ell_i^{(t)}\Big)\la\wb_j^{(t)},\vb_s\ra^2 \\
        &\le \la\wb_j^{(t)},\vb_s\ra+\frac{3\eta\beta_s^3}{N}\Big(\sum_{i\in S_1}\ell_i^{(t)}\Big)\la\wb_j^{(t)},\vb_s\ra^2 \\
        &\le \la\wb_j^{(t)},\vb_s\ra+\tilde\Theta(\eta)\beta_s^3\cdot\frac{\sum_{i\in S_1}g_1(t)}{N}\cdot \la\wb_j^{(t)},\vb_s\ra^2\\
        &= \la\wb_j^{(t)},\vb_s\ra+\tilde\Theta(\eta)\beta_s^3\hat{\alpha}\cdot g_1(t)\cdot \la\wb_j^{(t)},\vb_s\ra^2,
    \end{align*}
    where the first inequality holds due to $0<\ell_i^{(t)}<1$, the second inequality holds due to Lemma~\ref{lemma:large group derivative}, and the last equality holds due to $|S_1|/N=\hat{\alpha}$. Then, for the lower bound, we consider the same bound for $i\in S_1$ in Lemma~\ref{lemma:large group derivative} and obtain
    \begin{align*}
        \la\wb_j^{(t+1)},\vb_s\ra        &=\la\wb_j^{(t)},\vb_s\ra+\frac{3\eta\beta_s^3}{N}\Big(\sum_{i\in S_1}\ell_i^{(t)}-\sum_{i\in S_2}\ell_i^{(t)}\Big)\la\wb_j^{(t)},\vb_s\ra^2 \\
        &\ge  \la\wb_j^{(t)},\vb_s\ra+\tilde\Theta(\eta)\beta_s^3\Big(\hat{\alpha}\cdot g_1(t) - \sum_{i\in S_2}\ell_i^{(t)}/N \Big)\la\wb_j^{(t)},\vb_s\ra^2.
    \end{align*}
\end{proof}
Similarly, we have the update for the core feature as below.
\begin{lemma}[Core feature update] \label{lemma:core update all}
    For all $t\ge 0$ and $j\in[J]$, the core feature update is
    \begin{align*}
        \la\wb_j^{(t+1)},\vb_c\ra=\la\wb_j^{(t)},\vb_c\ra+\frac{3\eta\beta_c^3}{N}\Big(\sum_{i=1}^N\ell_i^{(t)}\Big)\la\wb_j^{(t)},\vb_c\ra^2,
    \end{align*}
    which gives
    \begin{align*}
\tilde\Theta(\eta)\beta_c^3\cdot\hat{\alpha}g_1(t)\cdot \la\wb_j^{(t)},\vb_c\ra^2 &\le \la\wb_j^{(t+1)},\vb_c\ra- \la\wb_j^{(t)},\vb_c\ra \\
        &\le \tilde\Theta(\eta)\beta_c^3\cdot\Big(\hat{\alpha}g_1(t)+\sum_{i\in S_2}\ell_i^{(t)}/N\Big)\cdot \la\wb_j^{(t)},\vb_c\ra^2.
    \end{align*}
\end{lemma}
\begin{proof}
    The core feature update is obtained by using the gradient update of $\Wb^{(t)}$ and plugging in \eqref{eq:v_c grad update}:
    \begin{align*}
        \la\wb_j^{(t+1)},\vb_c\ra
        &=\la\wb_j^{(t)}-\eta \nabla_{\wb_j}\cL(\Wb^{(t)}),\vb_c\ra\\
        &=\la\wb_j^{(t)},\vb_c\ra+\frac{3\eta\beta_c^3}{N}\Big(\sum_{i=1}^N\ell_i^{(t)}\Big)\la\wb_j^{(t)},\vb_c\ra^2.
    \end{align*}    
    We prove for the lower bound, 
    \begin{align*}
        \la\wb_j^{(t+1)},\vb_c\ra
        &=\la\wb_j^{(t)},\vb_c\ra+\frac{3\eta\beta_c^3}{N}\Big(\sum_{i=1}^N\ell_i^{(t)}\Big)\la\wb_j^{(t)},\vb_c\ra^2 \\
        &\ge \la\wb_j^{(t)},\vb_c\ra+\frac{3\eta\beta_c^3}{N}\Big(\sum_{i\in S_1}\ell_i^{(t)}\Big)\la\wb_j^{(t)},\vb_c\ra^2 \\
        &\ge \la\wb_j^{(t)},\vb_c\ra+\tilde\Theta(\eta)\beta_c^3\hat{\alpha}g_1(t)\cdot \la\wb_j^{(t)},\vb_c\ra^2, 
    \end{align*}
    where the first inequality holds due to $0<\ell_i^{(t)}<1$ and the second inequality holds due to Lemma~\ref{lemma:large group derivative}. And for the upper bound. we similarly have
    \begin{align*}
        \la\wb_j^{(t+1)},\vb_c\ra
        &=\la\wb_j^{(t)},\vb_c\ra+\frac{3\eta\beta_c^3}{N}\Big(\sum_{i=1}^N\ell_i^{(t)}\Big)\la\wb_j^{(t)},\vb_c\ra^2 \\
        &\le\la\wb_j^{(t)},\vb_c\ra+\tilde\Theta(\eta)\beta_c^3\cdot\Big(\hat{\alpha}g_1(t)+\sum_{i\in S_2}\ell_i^{(t)}\Big)\cdot \la\wb_j^{(t)},\vb_c\ra^2. 
    \end{align*}
\end{proof}
Note that $\la\wb_j^{(t+1)},\vb_c\ra$ is non-decreasing from the lower bound of Lemma~\ref{lemma:core update all}. As $\wb_j^{(0)}\sim\cN(0,\sigma_0^2\Ib_d)$ are initialized with small $\sigma_0$, the sigmoid terms $\ell_{i}^{(t)}$ are large in the initial iterations. And while $l_i^{(t)}$ remains large for $i\in S_1$, we have $g_1(t)=\Theta(1)$ as similar as in \cite{jelassi2022towards}. Therefore, $\la\wb_j^{(t+1)},\vb_s\ra$ is also non-decreasing since $\hat{\alpha}\cdot\Theta(1)-\sum_{i\in S_2}l_i^{(t)}/N\ge 2\hat{\alpha}-1>0$ for $l_i^{(t)}<1$ and $\hat\alpha>1/2$. Eventually, $g_1(t)$ becomes small at a time $T_0>0$. We now consider a simplified version of the above lemma in this early training stage. 
\begin{lemma}[Spurious feature update in early iterations]
    \label{lemma:spurious update early}
    Let $T_0>0$ be such that $\max_{j\in[J]}\la\wb_j^{(T_0)},\vb_s\ra\ge\tilde\Omega(1/\beta_s)$. For $t\in[0,T_0]$, the spurious feature update has the following bound
    \begin{align*}
       \tilde\Theta(\eta)\beta_s^3(2\hat{\alpha} - 1)\cdot\la\wb_j^{(t)},\vb_s\ra^2
       \le\la\wb_j^{(t+1)},\vb_s\ra- \la\wb_j^{(t)},\vb_s\ra\le \tilde\Theta(\eta)\beta_s^3\hat{\alpha}\cdot\la\wb_j^{(t)},\vb_s\ra^2.
    \end{align*}
\end{lemma}
\begin{proof}
    Let  $T_0>0$ be such that either $\max_{j\in[J]}\la\wb_j^{(T_0)},\vb_s\ra\ge\tilde\Omega(1/\beta_s)$ or $\max_{j\in[J]}\la\wb_j^{(T_0)},\vb_c\ra\ge\tilde\Omega(1/\beta_c)$. We will show later that the first condition will be first met and we have $\la\wb_j^{(t)},\vb_c\ra\le\tilde\Omega(1/\beta_c)$ for all $j\in[J]$ and $t\in[0,T_0]$. 
    
    Recall that  $g_1(t)=\text{sigmoid}\big(-\sum_{j\in[J]}(\beta_c^3 \la\wb_j^{(t)},\vb_c\ra^3+\beta_s^3 \la\wb_j^{(t)},\vb_s\ra^3)\big)$. Then, for $t\in[0,T_0]$, we have
    \begin{align*}
        g_1(t) &= \frac{1}{1+\exp\big(\sum_{j\in[J]}(\beta_c^3 \la\wb_j^{(t)},\vb_c\ra^3+\beta_s^3 \la\wb_j^{(t)},\vb_s\ra^3)\big)} \\
        &\ge \frac{1}{1+\exp\big(\kappa + \kappa\big)} \\
        &= \frac{1}{1+\exp\big(\tilde{\Omega}(1)\big)}, 
    \end{align*}
    where the first inequality holds due to $\la\wb_s^{(t)},\vb_s\ra\le\kappa/(J^{1/3}\beta_s)$ and $\la\wb_s^{(t)},\vb_c\ra\le\kappa/(J^{1/3}\beta_c)$ for $t\in[0,T_0]$. Therefore, similar to \cite{jelassi2022towards}, we have $g_1(t)=\Theta(1)$ in the early iterations. Moreover, as $0<\ell_i^{(t)}<1$, we have $\sum_{i\in S_2}\ell_i^{(t)}/N < 1-\hat{\alpha}$. This implies the result in Lemma~\ref{lemma:spurious update all} as
     \begin{align*}
       \tilde\Theta(\eta)\beta_s^3(2\hat{\alpha} - 1)\la\wb_j^{(t)},\vb_s\ra^2
       \le\la\wb_j^{(t+1)},\vb_s\ra- \la\wb_j^{(t)},\vb_s\ra\le \tilde\Theta(\eta)\beta_s^3\hat{\alpha}\la\wb_j^{(t)},\vb_s\ra^2.
    \end{align*}
\end{proof}
And similarly, for the core feature, we have
\begin{lemma}[Core feature update in early iterations]
    \label{lemma:core update early}
    Let $T_0>0$ be such that $\max_{j\in[J]}\la\wb_j^{(T_0)},\vb_s\ra\ge\tilde\Omega(1/\beta_s)$. For $t\in[0,T_0]$, the core feature update has the following bound
    \begin{align*}
       \tilde\Theta(\eta)\beta_c^3\hat{\alpha} \cdot\la\wb_j^{(t)},\vb_c\ra^2
       \le\la\wb_j^{(t+1)},\vb_c\ra- \la\wb_j^{(t)},\vb_c\ra\le \tilde\Theta(\eta)\beta_c^3\cdot\la\wb_j^{(t)},\vb_c\ra^2.
    \end{align*}
\end{lemma}
\begin{proof}
    Let  $T_0>0$ be such that either $\max_{j\in[J]}\la\wb_j^{(T_0)},\vb_s\ra\ge\tilde\Omega(1/\beta_s)$ or $\max_{j\in[J]}\la\wb_j^{(T_0)},\vb_c\ra\ge\tilde\Omega(1/\beta_c)$. Again, with $g_1(t)=\Theta(1)$ and $\sum_{i\in S_2}\ell_i^{(t)}/N < 1-\hat{\alpha}$ as shown in Lemma~\ref{lemma:spurious update early}, we can imply the result in Lemma~\ref{lemma:core update all} as 
    \begin{align*}
       \tilde\Theta(\eta)\beta_c^3\hat{\alpha} \cdot\la\wb_j^{(t)},\vb_c\ra^2
       \le\la\wb_j^{(t+1)},\vb_c\ra- \la\wb_j^{(t)},\vb_c\ra\le \tilde\Theta(\eta)\beta_c^3\cdot\la\wb_j^{(t)},\vb_c\ra^2, 
    \end{align*}
    which completes the proof. 
\end{proof}
With the updates of the spurious and core feature in the early iterations, we can now show with the following lemma that GD will learn the spurious feature very quickly while hardly learning the core feature.
\begin{lemma}\label{lemma:core upper bounded}
    Let $T_0$ be the iteration number that $\max_{j\in[J]}\la\wb_j^{(t)},\vb_s\ra$ reaches $\tilde\Omega(1/\beta_s)=\tilde\Theta(1)$. Then, we have for all $t\le T_0$, it holds that $\max_{j\in[J]}\la\wb_j^{(t)},\vb_c\ra=\tilde{O}(\sigma_0)$.
\end{lemma}
\begin{proof}
    Consider the following from Lemma~\ref{lemma:spurious update early} and Lemma~\ref{lemma:core update early},
    \begin{align*}
        &\la\wb_j^{(t+1)},\vb_c\ra- \la\wb_j^{(t)},\vb_c\ra\le \tilde\Theta(\eta)\beta_c^3\cdot\la\wb_j^{(t)},\vb_c\ra^2\\
        &\la\wb_j^{(t+1)},\vb_s\ra- \la\wb_j^{(t)},\vb_s\ra\ge \tilde\Theta(\eta)\beta_s^3(2\hat{\alpha} - 1)\la\wb_j^{(t)},\vb_s\ra^2.
    \end{align*}
    Recall that we initialize the weights as $\wb_j^{(0)}\sim\cN(\zero,\sigma_0^2)$. We have $\la\wb_j^{(0)},\vb_c\ra\sim\cN(0,\sigma_0^2)$ and $\la\wb_j^{(0)},\vb_s\ra\sim\cN(0,\sigma_0^2)$. For the weights have small initialization with $\sigma_0=\polylog(d)/d$, we have $O(\la\wb_j^{(0)},\vb_c\ra)=O(\la\wb_j^{(0)},\vb_s\ra)$. Therefore, for $\beta_c^3=o(1)$ and $\beta_s^3(2\hat{\alpha}-1)=\Theta(1)$, we call Lemma~\ref{lemma:tensor power update} and get 
    \[\la\wb_j^{(T_0)},\vb_c\ra\le O(\la\wb_j^{(0)},\vb_s\ra)=\tilde{O}(\sigma_o)\]
    for all $j\in[J]$. 
\end{proof}
Given the above lemma, we can conclude that the condition $\max_{j\in[J]}\la\wb_j^{(T_0)},\vb_s\ra\ge\tilde\Omega(1/\beta_s)$ will be first met. And therefore, $T_0$ is such that $\max_{j\in[J]}\la\wb_j^{(T_0)},\vb_s\ra\ge\tilde\Omega(1/\beta_s)$. 
\begin{theorem}[Restatement of Theorem~\ref{lemma:stage1 spurous}]
    Consider the training dataset $S=\{(\xb_i,y_i)\}_{i=1}^{N}$ that follows the distribution in Definition~\ref{def:data_distribution}. Consider the two-layer nonlinear CNN model as in~\eqref{eq:model} initialized with $\Wb^{(0)}\sim \cN(0,\sigma_0^2)$. After training with GD in \eqref{eq:GD update} for $T_0=\tilde{\Theta}\big(1/(\eta\beta_s^3\sigma_0)\big)$ iterations, for all $j\in [J]$ and $t\in[0,T_0)$, we have 
    \begin{align}
        \tilde\Theta(\eta)\beta_s^3(2\hat{\alpha} - 1)\cdot\la\wb_j^{(t)},\vb_s\ra^2
        &\le\la\wb_j^{(t+1)},\vb_s\ra- \la\wb_j^{(t)},\vb_s\ra\le \tilde\Theta(\eta)\beta_s^3\hat{\alpha}\cdot\la\wb_j^{(t)},\vb_s\ra^2,\label{eq:spurious bounds app}\\
        \tilde\Theta(\eta)\beta_c^3\hat{\alpha} \cdot\la\wb_j^{(t)},\vb_c\ra^2
        &\le\la\wb_j^{(t+1)},\vb_c\ra- \la\wb_j^{(t)},\vb_c\ra\le \tilde\Theta(\eta)\beta_c^3\cdot\la\wb_j^{(t)},\vb_c\ra^2\label{eq:core bounds app}.
    \end{align} 
    After training for $T_0$ iterations, with high probability, the learned weight has the following properties: (1) it learns the spurious feature $\vb_s$: $\max_{j\in[J]}\la\wb_j^{(T)},\vb_s\ra \ge \tilde{\Omega}(1/\beta_s)$; (2) it does not learn the core feature $\vb_c$: $\max_{j\in[J]}\la\wb_j^{(T)},\vb_c\ra = \tilde{\cO}(\sigma_0)$. 
\end{theorem}
\begin{proof}
    The updates directly follow the results from Lemma~\ref{lemma:spurious update all} and Lemma~\ref{lemma:core update all}. And the result for $\max_{j\in[J]}\la\wb_j^{(t)},\vb_c\ra$ follows Lemma~\ref{lemma:core upper bounded}. It remains to calculate the time $T_0$. 
    With Lemma~\ref{lemma:tensor power update same seq}, we consider the sequence for $\max_{j\in[J]}\la\wb_j^{(t+1)},\vb_s\ra$, where by Lemma~\ref{lemma:spurious update early}, 
    \begin{align*}
        &\la\wb_j^{(t+1)},\vb_s\ra \le \la\wb_j^{(t)},\vb_s\ra + \tilde\Theta(\eta)\beta_s^3\hat{\alpha}\cdot\la\wb_j^{(t)},\vb_s\ra^2, \\
        &\la\wb_j^{(t+1)},\vb_s\ra \ge \la\wb_j^{(t)},\vb_s\ra + \tilde\Theta(\eta)\beta_s^3(2\hat{\alpha}-1)\cdot\la\wb_j^{(t)},\vb_s\ra^2.
    \end{align*}
    As $\la\wb_j^{(t)},\vb_s\ra$ is non-decreasing in early iterations and with high probability, there exists an index $j$ such that $\la\wb_j^{(0)},\vb_s\ra\ge 0$. Among all the possible indices $i\in[J]$ that are initialized to have positive inner product with $\vb_s$, we focus on the max index  $r=\argmax_{j\in[J]}\la\wb_j^{(0)},\vb_s\ra$. Then with $v=\tilde\Theta(1/\beta_s)$ in Lemma~\ref{lemma:tensor power update same seq}, we will have $T_0$ as
    \[T_0=\frac{\tilde\Theta(1)}{\eta\alpha^3\sigma_0}+\frac{\tilde\Theta(1)\hat{\alpha}}{2\hat{\alpha}-1}\bigg\lceil\frac{-\log\big(\sigma_0\beta_s\big)}{\log(2)}\bigg\rceil.\]
\end{proof}

\section{Proof of Lemma~\ref{lemma:motivation}}\label{appendix:motivation}
\begin{lemma}[Restatement of Lemma~\ref{lemma:motivation}]
   Given the balanced training dataset $S^0=\{(\xb_i,y_i,a_i)\}_{i=1}^{N_0}$ with $\hat{\alpha}=1/2$ as in Definition~\ref{def:data_distribution} and CNN as in \eqref{eq:model}. The gradient on $\vb_s$ will be $0$ from the beginning of training.
\end{lemma}
\begin{proof}
    With Lemma~\ref{lemma:grad}, the projection of the gradient on $\vb_s$ in the initial iteration ($t<T_0$) is
    \begin{align*}
    \la\nabla_{\wb_j}\cL(\Wb^{(t)}),\vb_s\ra&=-\frac{3\beta_s^3}{N}\Big(\sum_{i\in S_1}\ell_i^{(t)}-\sum_{i\in S_2}\ell_i^{(t)}\Big)\cdot\la\wb_j,\vb_s\ra^2\\
    &= \Theta\bigg(\frac{\beta_s^3}{N}\bigg) \Big(|S_1|-|S_2|\Big)\\
    &=0,
    \end{align*}
    where the first equality is due to $\ell_i^{(t)}=\Theta(1)$ in the initial iterations and the second equality is due to $\hat{\alpha}=0.5$. 
\end{proof}

\section{Auxiliary Lemmas}
\begin{lemma}[Lemma C.20, \citealt{allen2020towards}]\label{lemma:tensor power update}
Let $\{x_{t},y_{t}\}_{t=1,..}$ be two positive sequences that satisfy
\begin{align*}
    x_{t+1}&\geq x_{t}+ \eta \cdot A x_{t}^{2},\\
    y_{t+1}&\leq y_{t}+ \eta\cdot B y_{t}^{2},
\end{align*}
for some $A=\Theta(1)$ and $B=o(1)$. Suppose $y_0=O(x_0)$ and $\eta<O(x_0)$, and for all $C\in[X_0,O(1)]$, let $T_x$ be the first iteration such that $x_t\ge C$. Then, we have $T_x\eta=\Theta(x_0^{-1})$ and
\[y_{T_x}\le O(x_0).\]
\end{lemma}

\begin{lemma}[Lemma K.15, \citealt{jelassi2022towards}]\label{lemma:tensor power update same seq}
Let $\{z_t\}_{t=0}^T$ be a positive sequence defined by the following recursions
\begin{align*}
    z_{t+1} \ge z_t+m(z_t)^2,\\
    z_{t+1} \le z_t+M(z_t)^2,
\end{align*}
where $z_0>0$ is the initialization and $m,M>0$ are some constants. Let $v>0$ such that $z_0\le v$. Then, the time $t_0$ such that $z_t\ge v$ for all $t\ge t_0$ is
\[t_0=\frac{3}{mz_0}+\frac{8M}{m}\bigg\lceil\frac{\log(v/z_0)}{\log(2)}\bigg\rceil.\]
\end{lemma}

We make the following assumptions for every $t\le T$ as the same in \citep{jelassi2022towards}.
\begin{lemma}[Induction hypothesis D.1, \citealt{jelassi2022towards}]\label{lemma:induction noise}
    Throughout the training process using GD for $t\le T$, we maintain that, for every $i\in S_1$ and $j\in[J]$, 
    \begin{align}
        |\la\wb_j^{(t)},\bxi_i\ra|\le\tilde{O}(\sigma_0\sigma\sqrt{d}).
    \end{align}
\end{lemma}

\begin{lemma}\label{lemma:large group derivative}
    For $i\in S_1$, we have $\ell_i^{(t)}=\Theta(1)g_1(t)$, where 
    \[g_1(t)=\text{sigmoid}\big(-\sum_{j\in[J]}(\beta_c^3 \la\wb_j^{(t)},\vb_c\ra^3+\beta_s^3 \la\wb_j^{(t)},\vb_s\ra^3)\big).\]
\end{lemma}
\begin{proof}
    Given $i\in S_1$, we have from \eqref{eq:data derivative} that
    \begin{align}
    \ell^{(t)}_i &= 
    \text{sigmoid}\bigg(\sum^J_{j=1}-\beta_c^3\la\wb_j,\vb_c\ra^3-\beta_s^3\la\wb_j,\vb_s\ra^3-y_i\la\wb_j,\bxi_i\ra^3\bigg)\nonumber\\
    &=1\bigg/\bigg(1+\exp\Big(\sum^J_{j=1}\beta_c^3\la\wb_j,\vb_c\ra^3+\beta_s^3\la\wb_j,\vb_s\ra^3+y_i\la\wb_j,\bxi_i\ra^3\Big)\bigg).\label{eq:aux l_i^t}
\end{align}
Recall induction hypothesis~\ref{lemma:induction noise}, we have the following for $i\in S_1$,
\begin{align}
    &|y_i\la\wb_j^{(t)},\bxi_i\ra|\le\tilde{O}(\sigma_0\sigma\sqrt{d})\nonumber\\
    &\iff -\tilde{O}(\sigma_0\sigma\sqrt{d})\le y_i\la\wb_j^{(t)},\bxi_i\ra \le\tilde{O}(\sigma_0\sigma\sqrt{d}),\label{eq:aux noise abs y_i bound}
\end{align}
where $|y_i|=1$. Plug \eqref{eq:aux noise abs y_i bound} back into \eqref{eq:aux l_i^t}, we get
\begin{align*}
    e^{-\Tilde{O}(\sigma_0\sigma\sqrt{d})^3}g_1(t) \le \ell_i^{(t)} \le e^{\Tilde{O}(\sigma_0\sigma\sqrt{d})^3}g_1(t).
\end{align*}
With our parameter setting, we have $\tilde{O}(\sigma_0\sigma\sqrt{d})=\tilde{O}(\sigma_0)=\tilde{O}(\text{polylog}(d)/d)$. Therefore, $e^{\pm\Tilde{O}(\sigma_0\sigma\sqrt{d})^3}=\Theta(1)$.
\end{proof}

\end{document}